\documentclass[graybox]{svmult}

\usepackage{type1cm}        
\usepackage{makeidx}         
\usepackage{graphicx}        
\usepackage{multicol}        
\usepackage[bottom]{footmisc}
\usepackage{nicefrac}

\usepackage{newtxtext}       %
\usepackage[varvw]{newtxmath}       
\usepackage{caption}

\usepackage{pgfplots}
\pgfplotsset{compat=newest}
\usetikzlibrary{plotmarks}
\usetikzlibrary{arrows.meta}
\usepgfplotslibrary{patchplots}
\usepackage{grffile}
\newlength\fwidth

\newlength\height
\newlength\width

\newcommand*{\MAT}[1]  {\ensuremath{\boldsymbol{#1}}}

\newcommand*{\HPD}  {\ensuremath{\mathcal{H}^{++}}}

\DeclareMathOperator{\Diff}{D}
\DeclareMathOperator{\grad}{grad}

\DeclareMathOperator{\symm}{Herm}
\DeclareMathOperator{\tr}{Tr}

\DeclareMathOperator*{\argmin}{arg\,min}
\DeclareMathOperator*{\minimize}{minimize}
\DeclareMathOperator*{\reel}{\mathfrak{Re}}
\DeclareMathOperator*{\vech}{vech}

\newcommand*{\realSpace}{\ensuremath{\mathbb{R}}}
\newcommand*{\complexSpace}{\ensuremath{\mathbb{C}}}

\newcommand*{\expectation}[1]{\text{E}\left[#1\right]}

\newcommand*{\nfeatures}{\ensuremath{p}}
\newcommand*{\nsamples}{\ensuremath{n}}
\newcommand*{\nbatches}{\ensuremath{m}}
\newcommand*{\nclasses}{\ensuremath{z}}

\newcommand*{\ambientSpace}[1]{\ensuremath{\mathcal{H}_{#1}}}
\newcommand*{\manifold}[1]{\ensuremath{\mathcal{H}^{++}_{#1}}}
\newcommand*{\tangentSpace}[2]{\ensuremath{T_{#2}\mathcal{H}^{++}_{#1}}}
\newcommand*{\point}{\ensuremath{\mathbf{\Sigma}}}
\newcommand*{\pointBis}{\ensuremath{\mathbf{\hat{\Sigma}}}}
\newcommand*{\tangentVector}{\ensuremath{\boldsymbol{\xi}}}
\newcommand*{\tangentVectorBis}{\ensuremath{\boldsymbol{\eta}}}
\newcommand*{\tangentVectorTer}{\ensuremath{\boldsymbol{\nu}}}

\newcommand*{\metric}[3]{\ensuremath{\langle#2,#3\rangle_{#1}}}
\newcommand*{\metricCES}[3]{\ensuremath{\reel(\alpha\tr(#1^{-1}#2#1^{-1}#3) + \beta\tr(#1^{-1}#2)\tr(#1^{-1}#3))}}

\newcommand*{\squaredist}[2]{\ensuremath{\delta^2(#1,#2)}}

\newcommand*{\metricGeneric}[3]{\ensuremath{\prec#2,#3\succ_{#1}}}
\newcommand*{\squaredistGeneric}[2]{\ensuremath{\mathfrak{d}^2(#1,#2)}}
\newcommand*{\logGeneric}{\ensuremath{\mathfrak{log}}}

\newcommand*{\errVec}{\ensuremath{\boldsymbol{\varepsilon}}}

\newcommand*{\dataSpace}[1]{\ensuremath{\complexSpace^{#1}}}
\newcommand*{\data}{\ensuremath{\mathbf{x}}}
\newcommand*{\Data}{\ensuremath{\mathbf{X}}}
\newcommand*{\AugmentedData}{\ensuremath{\mathbf{\tilde{X}}}}

\newcommand*{\distribution}[2]{\ensuremath{\text{C-CES}(\mathbf{0},#1,#2)}}
\newcommand*{\densityGenerator}{\ensuremath{g}}
\newcommand*{\loglikelihood}{\ensuremath{\mathcal{L}_g}}

\newcommand*{\costfunction}{\ensuremath{L}}

\newcommand*{\estimator}{\ensuremath{\mathbf{\hat{\Sigma}}}}

\newcommand*{\labels}{\ensuremath{y}}
\newcommand*{\classifier}{\ensuremath{\mathcal{C}}}

\newcommand*{\dataset}{\ensuremath{\mathcal{T}}}
\newcommand*{\train}{\ensuremath{\mathcal{T}_{\textup{train}}}}
\newcommand*{\test}{\ensuremath{\mathcal{T}_{\textup{test}}}}
\newcommand*{\nbatchestrain}  {{\ensuremath{m_{\textup{train}}}}}
\newcommand*{\nbatchestest}  {{\ensuremath{m_{\textup{test}}}}}
\newcommand*{\classCenter}{\ensuremath{\mathbf{\bar{\Sigma}}}}

\newcommand*{\mtest}  {{\ensuremath{m_\text{test}}}}

\usepackage{xcolor}

\usepackage{pgfplotsthemetol}
\definecolor{TolDarkPurple}{HTML}{332288}
\definecolor{TolDarkBlue}{HTML}{6699CC}
\definecolor{TolLightBlue}{HTML}{88CCEE}
\definecolor{TolLightGreen}{HTML}{44AA99}
\definecolor{TolDarkGreen}{HTML}{117733}
\definecolor{TolDarkBrown}{HTML}{999933}
\definecolor{TolLightBrown}{HTML}{DDCC77}
\definecolor{TolDarkRed}{HTML}{661100}
\definecolor{TolLightRed}{HTML}{CC6677}
\definecolor{TolLightPink}{HTML}{AA4466}
\definecolor{TolDarkPink}{HTML}{882255}
\definecolor{TolLightPurple}{HTML}{AA4499}

\definecolor{mDarkBrown}{HTML}{604c38}
\definecolor{mDarkTeal}{HTML}{23373b}
\definecolor{mLightBrown}{HTML}{EB811B}
\definecolor{mLightGreen}{HTML}{14B03D}

\colorlet{mLightTeal}{mDarkTeal!75}
\colorlet{mAltBrown}{mLightBrown!80!mDarkBrown!90!red}
\colorlet{mLightBlue}{TolDarkBlue!70!mDarkTeal}
\colorlet{mLightYellow}{TolLightBrown!70!mLightBrown}
\colorlet{mDarkOrange}{mLightBrown!40!red!75!mDarkTeal}
\definecolor{myblue}{HTML}{1f77b4}
\definecolor{myorange}{HTML}{ff7f0e}
\definecolor{myred}{HTML}{d62728}

\colorlet{mCiteColor}{mLightBlue}

\newtheorem{Remark}{Remark}
\newtheorem{Example}{Example}[theorem]

\usepackage[ruled,vlined]{algorithm2e}

\makeindex             

\begin{document}

\title*{The Fisher-Rao geometry of CES distributions}
\author{Florent Bouchard, Arnaud Breloy, Antoine Collas, Alexandre Renaux, Guillaume Ginolhac}
\authorrunning{F. Bouchard, A. Breloy, A. Collas, A. Renaux, G. Ginolhac} 
\institute{
Florent Bouchard \at Universit\'e Paris-Saclay, CNRS, CentraleSup\'elec, L2S, \\ \email{florent.bouchard@centralesupelec.fr}
\and Arnaud Breloy \at LEME, Universit\'e Paris Nanterre, \\ \email{arnaud.breloy@parisnanterre.fr}
\and Antoine Collas \at Universit\'e Paris-Saclay, Inria, CEA, \\ \email{antoine.collas@inria.fr}
\and Alexandre Renaux \at Universit\'e Paris-Saclay, CNRS, CentraleSup\'elec, L2S, \\ \email{alexandre.renaux@centralesupelec.fr}
\and Guillaume Ginolhac \at LISTIC, Universit\'e Savoie Mont-Blanc, \\ \email{guillaume.ginolhac@univ-smb.fr}
}
%
%
\maketitle

\abstract{
When dealing with a parametric statistical model, a Riemannian manifold can naturally appear by endowing the parameter space with the Fisher information metric. 
The geometry induced on the parameters by this metric is then referred to as the Fisher-Rao information geometry.
Interestingly, this yields a point of view that allows for leveraging many tools from differential geometry. 
After a brief introduction about these concepts, we will present some practical uses of these geometric tools in the framework of elliptical distributions. 
This second part of the exposition is divided into three main axes: Riemannian optimization for covariance matrix estimation, Intrinsic Cram\'er-Rao bounds, and classification using Riemannian distances.
}

\section{Introduction: from CES distributions to information geometry}
\label{sec:intro}

This section starts with reminders on complex elliptically symmetric distributions (CES)\footnote{
Note that this chapter considers the case where the data and covariance matrix can be complex-valued for the sake of generality.
However, we focus solely on the circular case (referred to as C-CES in the background chapter). 
Hence, most of the presented results can also be obtained in the real-valued case (RES) with proper adjustments. 
}.
This part is concluded by introducing the Fisher information matrix of this model, which acts as a transition point to information geometry.
Indeed, the Fisher information matrix actually represents a metric that induces an inherent geometry for statistical models, which is referred to as the Fisher-Rao information geometry.
In the specific case of CES, this will yield a particular geometry for the space of covariance matrices.
After evidencing this transition point, this section concludes by outlining the rest of the chapter.

\subsection{Reminders on CES distributions}

\label{sec:background_ces}


Circular complex elliptically symmetric (C-CES) distributions \cite{kai1990generalized} refer to a large family of multivariate distributions.
Very comprehensive and detailed reviews on the topic can be found in the references \cite{ollila2011complex, ollila2012complex}, and of course, the background chapter of this book.
A vector $\data\in\dataSpace{\nfeatures}$ follows a centered (zero-mean) C-CES distribution, denoted $\data\sim\distribution{\point}{\densityGenerator}$, if it admits the following stochastic representation:
\begin{equation}
\label{eq:ces_repr_thm}
\data {=}_{d} \sqrt{\mathcal{Q}} ~\point^{\frac{1}{2}} ~\mathbf{u}
,
\end{equation}
where:
\begin{itemize}
    
    \item[$\bullet$]  The notation ${=}_{d}$ means that random variables on both sides have the same cumulative distribution function.
    
    \item[$\bullet$] The vector $\mathbf{u}\in\dataSpace{\nfeatures}$ follows a uniform distribution on the complex unit sphere $\complexSpace\mathcal{S}^{\nfeatures} = \left\{ \mathbf{u}\in\dataSpace{\nfeatures} ~|~ \left\|\mathbf{u} \right\| = 1  \right\}$, denoted $\mathbf{u} \sim \mathcal{U} \left( \complexSpace\mathcal{S}^{\nfeatures} \right)$. 
    
    \item[$\bullet$] The scalar $\mathcal{Q}\in \realSpace^+$ is non-negative real random variable of probability density function $f_{\mathcal{Q}}$, independent of $\mathbf{u}$, and called the second-order modular variate (while $\sqrt{\mathcal{Q}}$ is called the modular variate).

    \item[$\bullet$] The matrix $\point^{\frac{1}{2}}\in \complexSpace^{\nfeatures\times\nfeatures}$ is a factorization of the scatter matrix $\point = \point^{\frac{1}{2}}\point^{\frac{H}{2}}$.
    If the covariance matrix of $\data$ exists, it is proportional to the scatter matrix, i.e., $\expectation{\data \data^H} \propto \point$.
    If we then choose the normalization convention $\expectation{\mathcal{Q}}=1$, these two matrices are equal.
    Thus we will abusively refer to the scatter matrix $\point$ as the covariance matrix, as it is a more familiar terminology. 
\end{itemize}
We focus only on the absolutely continuous case where the covariance matrix $\point$ is full rank (cf. Section 2.3 of the background chapter).
In this case, the probability density function of $\data$ is given as
\begin{equation}
f_{\data} \left(  \data | \point \right)  \propto |\point|^{-1} \densityGenerator \left( ~ \data^H \point^{-1} \data ~ \right)
,
\end{equation}
where the function $\densityGenerator:\realSpace^+\to\realSpace^+$ is called the density generator.
The density generator satisfies the finite moment condition
$\delta_{\nfeatures,\densityGenerator} = \int_0^{\infty} t^{\nfeatures-1}\densityGenerator(t) d t < \infty$.
This function $\densityGenerator$ is directly related to the probability density function of the second-order modular variate by the relation
\begin{equation}
\label{eq:pdf_2modular_variate}
f_{\mathcal{Q}}\left( \mathcal{Q} \right) = \delta_{\nfeatures,\densityGenerator}^{-1} \mathcal{Q}^{\nfeatures-1} \densityGenerator \left( \mathcal{Q} \right)
.
\end{equation}
Given a $\nsamples$-sample $\{\data_i\}_{i=1}^{\nsamples}$, assumed to be independent and identically distributed (iid) from $\data\sim\distribution{\point}{\densityGenerator}$, its log-likelihood is given as:
\begin{equation}
\label{eq:log_likelihood}
\loglikelihood\left(\{ \data_i \}_{i=1}^{\nsamples}  | \point \right)   =   \sum_{i=1}^{\nsamples} \log \left(  \densityGenerator \left( \data_i^H  \point^{-1} \data_i \right) \right) - \nsamples \log |\point|
.
\end{equation}
%



The C-CES model being defined, we move to the notion of information brought by the Fisher information matrix: intuitively, the more the sample set $ \{ \data_i \}_{i=1}^{\nsamples} $ depends on $\point$, the more sampling from the likelihood \eqref{eq:log_likelihood} (increasing $\nsamples$) will reveal information about $\point$.
The score vector is a tool that will help in quantifying this notion of information:
to define this quantity, we now consider a parameterization of the covariance matrix through a real-valued vector $\boldsymbol{\nu}$ of appropriate dimension\footnote{
In this chapter, $\point$ is not assumed to have a specific structure, so $\boldsymbol{\nu}$ is typically of dimension $\nfeatures^2$ and stores the entries of the diagonal and upper triangle of the covariance matrix (where the coordinates are split in terms of real and imaginary part).
However, the definition extends to any parameterization, e.g., from a choice of decomposition in the case of structured matrices \cite{meriaux2019robust}.
}, denoted $\point(\boldsymbol{\nu})$.
The score vector $ \mathbf{s} $ is then defined entry-wise as
\begin{equation}
    \left[ \mathbf{s} \right]_j = \frac{\partial \loglikelihood \left(   \{ \data_i \}_{i=1}^{\nsamples}  | \point(\boldsymbol{\nu}) \right)   }{\partial \nu_j },
\end{equation}
which therefore reflects the variation of the log-likelihood of the sample set $ \{ \data_i \}_{i=1}^{\nsamples} $ with respect to the parameter $\nu_j$.
Under mild regularity conditions (satisfied by $\loglikelihood$ in our case), this vector has zero mean, i.e., $\expectation{\mathbf{s}} = \mathbf{0}$.
However, its covariance matrix is a fundamental quantity referred to as the Fisher information matrix, denoted $\mathbf{F}$, and defined as: 
\begin{equation} \label{eq:fisher_matrix_def}
    \left[ \mathbf{F}\right]_{j,k} 
    =
    \expectation{
        \frac{\partial \loglikelihood \left(   \{ \data_i \}_{i=1}^{\nsamples} | \point(\boldsymbol{\nu}) \right)   }{\partial \nu_j }
        \frac{\partial \loglikelihood \left( \{ \data_i \}_{i=1}^{\nsamples}  | \point(\boldsymbol{\nu}) \right)   }{\partial \nu_k }
    }.
\end{equation}
This matrix quantifies, on average, how much information about the vector $\boldsymbol{\nu}$ we can obtain from a sample set $\{\data_i\}_{i=1}^{\nsamples}$.
In practice, the entries of the Fisher information matrix $\mathbf{F}$ for centered C-CES can be obtained thanks to Slepian-Bangs type formula from \cite{besson2013fisher}, also presented in the Section 6.5 of the background chapter.
The latter is briefly recalled below using an alternate expression that is consistent whit the upcoming discussions:


\begin{theorem}(Fisher Information matrix of centered C-CES)\\
\label{thm:fisher_metric_ces}
Let $\point\overset{\text{def}}{=}\point(\boldsymbol{\nu})$ be a covariance matrix parameterized by the real-valued vector $\boldsymbol{\nu}$.
Let $\{\data_i\}_{i=1}^{\nsamples}$ be a $\nsamples$-sample of iid from $\data \sim \distribution{\point}{\densityGenerator}$.
The entries of the Fisher information matrix are
\begin{equation}
\label{eq:SB_metric_CES}
\left[ \mathbf{F}\right]_{j,k}  =
    n \alpha_{\densityGenerator}
    \tr(\point^{-1}\tangentVector_j\point^{-1}\tangentVector_k)
    + 
    n \beta_{\densityGenerator}
    \tr(\point^{-1}\tangentVector_j)\tr(\point^{-1}\tangentVector_k),
\end{equation}
with 
\begin{equation} \label{eq:tangent_xi}
    \tangentVector_j = \frac{\partial \point(\boldsymbol{\nu}) }{\partial \nu_j },
\end{equation}
and where the coefficients $\alpha_g$ and $\beta_g$ are defined by
\begin{equation}
\label{eq:coefficients_metric}
\begin{aligned}
&\alpha_g =    1   -  \frac{\mathbb{E} \left[ \mathcal{Q}^2 \phi' \left( \mathcal{Q}  \right) \right]  
}{\nfeatures(\nfeatures+1)}  
&~~~~\text{and}~~~~&
&\beta_g  =  \alpha_g -1,
\end{aligned}
\end{equation}
with $\phi(t) = g'(t)/g(t)$.
\end{theorem}
In the statistical signal processing community \cite{kay1993fundamentals}, the Fisher information matrix has been extensively leveraged thanks to the Cramér-Rao inequality:
\begin{equation}
    \mathbb{E} \left[ (\hat{\boldsymbol{\nu}}-\boldsymbol{\nu})
    (\hat{\boldsymbol{\nu}}-\boldsymbol{\nu})^{\top}  \right]
    \succcurlyeq
    \mathbf{F}^{-1}
    \quad
    \Rightarrow
    \quad 
    \left\| \hat{\boldsymbol{\nu}}-\boldsymbol{\nu} \right\|_F^2 \geq \tr( \mathbf{F}^{-1}),
\end{equation}
that yields a lower bound for the mean squared error of any unbiased estimator $\hat{\boldsymbol{\nu}}$ of
$\boldsymbol{\nu}$ (built from a set of observations $ \{ \data_i \}_{i =1}^{\nsamples} $).
On the other hand, the seminal work of Rao \cite{rao1945information, rao1992information} also discusses using the Riemannian geometry of the parameter space when the Fisher information matrix is used as a metric tensor.
The study of such spaces is now broadly referred to as the Fisher-Rao information geometry, which is introduced in the next section.
Before this, we conclude this brief reminder by the example of multivariate (Student's) $t$ distribution (also discussed with more details in Section 5.2 of the background chapter).

\Example{
The $t$-distribution with $d \in \mathbb{N}^*$ degrees of freedom is obtained for the C-CES representation $\data\sim\distribution{\point}{\densityGenerator_d}$ with
\begin{equation}
g_d (t) = \left( 1 + \frac{t}{d} \right)^{-(d+\nfeatures)} ,
\end{equation}
and the second-order modular variate is distributed as
$\mathcal{Q} {=}_{d}  {\mathbb{C}_{\mathcal{X}_\nfeatures^2}}/{\mathbb{C}_{\mathcal{X}_d^2/d} }$
where $\mathbb{C}_{\mathcal{X}_x^2}$ denotes the Chi-squared distribution with $x$ degrees of freedom. Hence $\mathcal{Q}$ follows a scaled $\mathcal{F}$-distribution.
We have
\begin{equation}
\label{eq:phi_mle_t}
\phi(t) = - \frac{d+\nfeatures}{  d +  t } ,
\end{equation}
and the expectation
\begin{equation}
\mathbb{E} \left[ \mathcal{Q}^2 \phi^2 \left( \mathcal{Q}  \right) \right]  
= \frac{(d+\nfeatures)\nfeatures(\nfeatures+1)}{d+\nfeatures+1}
,
\end{equation}
that allows to obtain the coefficients 
\begin{equation}
\alpha_g = \frac{d+\nfeatures}{d+\nfeatures+1}
~~~~~~~~\text{and}~~~~~~~~
\beta_g  = \frac{-1}{d+\nfeatures+1}.
\end{equation}
for the Fisher information metric as in Theorem \ref{thm:fim_metric}.
The $t$-distribution also encompasses the well known multivariate Gaussian model 
$\mathbf{x}\sim\mathcal{N}(\mathbf{0}, \mathbf{\Sigma})$
(of density generator $g_{\mathcal{N}}(t)=\exp(-t)$) as limit case when $d\rightarrow \infty$.
The corresponding Fisher information metric coefficient are then $\alpha_g = 1$ and $\beta_g  = 0$, which makes Theorem \ref{thm:fisher_metric_ces} coincide with the classical Slepian-Bangs formula \cite{slepian1954estimation, bangs1971array}.
}

\subsection{From the Fisher information matrix to information geometry}

\label{sec:from_fim_to_rg}

This section aims at linking notions of Riemannian geometry to the classical expression of the Fisher information matrix from Theorem \ref{thm:fisher_metric_ces}.
The goal is to shortly build the intuition on why the C-CES statistical model naturally induces a certain geometry for covariance matrices, while the corresponding framework will be presented in details in Section \ref{sec:RiemGeo}.
We first need to re-interpret the expression of the Fisher information matrix
from two key points:
\begin{itemize}
    \item 
    \textbf{Covariance matrices belong to the smooth manifold $\manifold{\nfeatures}$}\\
    The matrix $\point\overset{\text{def}}{=}\point(\boldsymbol{\nu})$ is a point in the space of covariance matrices, i.e., the set of $\nfeatures\times\nfeatures$ positive definite Hermitian matrices
    \begin{equation}
    \manifold{\nfeatures}
    = 
    \left\{ 
    \point
    \in 
    \ambientSpace{\nfeatures}:~
    ~\forall~\data\in \dataSpace{\nfeatures} \backslash \{ \mathbf{0} \},
    ~\data^H\point \data > 0
    \right\}
    ,
    \end{equation}
    where $\ambientSpace{\nfeatures}$ denotes the set of $\nfeatures\times\nfeatures$ Hermitian matrices.
    As it is an open of the linear space $\ambientSpace{\nfeatures}$, the space $\manifold{\nfeatures}$ is a smooth manifold.
    This means that it admits a differential structure, and notably, a tangent space at each point $\point$, denoted $\tangentSpace{\nfeatures}{\point}$.
    For any point $\point$, this tangent space $\tangentSpace{\nfeatures}{\point}$ turns out to be identifiable to be $\ambientSpace{\nfeatures}$ (which again, comes from the fact that $\manifold{\nfeatures}$  is an open subspace of $\ambientSpace{\nfeatures}$). 
    An abstract representation of these spaces is presented in Figure~\ref{fig:tangent_from_curves}. 

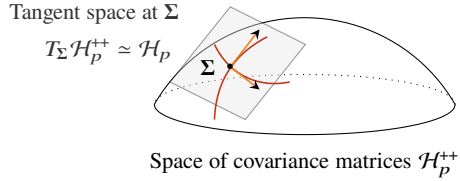
\begin{figure}
    \centering
    %
    \begin{center}
\begin{tikzpicture}[scale=2]
\draw  (1,0) arc (30:150:1.155)
      plot [smooth, domain=pi:2*pi] ({cos(\x r)},{0.2*sin(\x r)});
\draw [dotted] plot [smooth, domain=0:pi] ({cos(\x r)},{0.2*sin(\x r)});
\draw (0,-0.4) node {\small Space of covariance matrices $\mathcal{H}_p^{++}$};

\coordinate (x) at (-0.495,0.25);
\draw [fill=gray!20,opacity=0.4] (x)++(0.5,0.15) -- ++(-0.4,-0.5) -- ++(-0.5,+0.25) -- ++(+0.4,+0.5) -- cycle;
\draw [opacity=0.8] (-1.3,0.35) node {\small $T_{\mathbf{\Sigma}} \mathcal{H}_p^{++} \simeq \mathcal{H}_p $};
\draw [opacity=0.8] (-1.4,0.6) node {\small Tangent space at $\mathbf{\Sigma}$};
\draw (x)++(-0.25,0.1) node[anchor=north west] {\footnotesize$\mathbf{\Sigma}$};

\draw[draw=mDarkOrange,line width=0.6pt] (-0.25,0.45) to[bend right=30] (-0.6,-0.1);
\draw[draw=mLightBrown,line width=0.6pt,>=stealth,->] (x) -- (-0.3,0.5);

\draw[draw=mDarkOrange,line width=0.6pt] (-0.6,0.4) to[bend right=38] (-0.1,0.15);
\draw[draw=mLightBrown,line width=0.6pt,>=stealth,->] (x) -- (-0.3,0.1);

\draw (x) node {\tiny$\bullet$};

\end{tikzpicture}
\end{center}
    \caption{
    Space of covariance matrices represented as a smooth manifold.}
    \label{fig:tangent_from_curves}
\end{figure}

    \item \textbf{The Fisher information matrix represents an inner product on $\tangentSpace{\nfeatures}{\point}$}\\
    The entries of the Fisher information matrix of Theorem~\ref{thm:fisher_metric_ces} can be compactly denoted as
    $[\mathbf{F}]_{j,k}=\metric{\point}{\tangentVector_j}{\tangentVector_k}^{\textup{FIM}}$, whose expression is identified directly from~\eqref{eq:SB_metric_CES}.
    We then remark that the matrices $\tangentVector_j$ and $\tangentVector_k$ in~\eqref{eq:tangent_xi} are, in fact, elements of $\tangentSpace{\nfeatures}{\point}$.
    The expression in~\eqref{eq:SB_metric_CES} can thus be generalized to any pair of matrices $\tangentVector, \tangentVectorBis\in\tangentSpace{\nfeatures}{\point}$, which results in a bi-linear form, denoted $\metric{\point}{\cdot}{\cdot}^{\textup{FIM}}: \tangentSpace{\nfeatures}{\point}\times\tangentSpace{\nfeatures}{\point}\to\realSpace$.
    Because this bi-linear form is positive definite, it defines a metric, i.e., an inner product on the tangent space $\tangentSpace{\nfeatures}{\point}$.
    This inner product on $\tangentSpace{\nfeatures}{\point}$ is referred to as the Fisher information metric\footnote{
    Note that the Fisher information \textit{matrix} being obtained as $ \left[\mathbf{F}\right]_{j,k} = \metric{\point}{\tangentVector_j}{\tangentVector_k}^{\textup{FIM}}$, it is actually a matrix representation (metric tensor) of the Fisher information \textit{metric} $\metric{\point}{\cdot}{\cdot}^{\textup{FIM}}$ when the set $\{ \tangentVector_j \} $ is chosen as a basis of coordinates for the tangent space $\tangentSpace{\nfeatures}{\point}$.}.

\end{itemize}

\noindent

%

%

These two points being stated, we can last notice that the obtained Fisher information metric $\metric{\point}{\cdot}{\cdot}^{\textup{FIM}}$ varies smoothly with $\point$.
This enables the transition from statistical models to Riemannian geometry: the branch of differential geometry studying smooth manifolds endowed with smooth local inner products (referred to as Riemannian metrics).
Such framework indeed applies to parametric statistical models, as it allows us to investigate the geometry of the parameter space equipped with the Fisher information metric.
The resulting Riemannian geometry is generally referred to as the Fisher-Rao information geometry.
Back to our central example, we have presented enough elements to explicit that the title of this chapter ``The Fisher-Rao Geometry of CES distributions'' more precisely stands short for ``the Riemannian geometry of Hermitian positive definite matrices (covariance matrices) induced by the Fisher information metric of centered circular complex elliptically symmetric distributions'', which will be studied in the next sections.

\subsection{Outline of the chapter}

The previous section showed why an inherent geometry of the parameter space can naturally result from a statistical model.
Studying such geometry in detail requires introducing tools from the framework of Riemannian geometry, which is done in section \ref{sec:RiemGeo}.
The C-CES distributions will be used as an example throughout the exposition.
Hence, we will obtain most tools related to the Fisher-Rao Geometry of C-CES distributions: the Levi-Civita connection, the geodesics (and geodesic distance) between two covariance matrices, as well as the Riemannian exponential and logarithm mappings.

On a larger perspective, the second part of this chapter illustrates where tools obtained from the Fisher-Rao information geometry of C-CES can be leveraged within signal processing and machine learning tasks.
In details:
\begin{itemize}
    \item[$\bullet$] Section \ref{sec:Optim} addresses covariance matrix estimation problems, i.e., given a sample set $ \{ \data_i \}_{i=1}^{\nsamples} $, we infer $\point$ to perform a task (covariance analysis, filtering, metric learning, etc.).
    In this setup, we illustrate how the concepts of geodesic convexity and Riemannian optimization can be helpful in problems related to covariance matrix estimation. 

    \item[$\bullet$] 
    Still related to covariance matrix estimation problems, Section \ref{sec:iCRLB}
    presents how the statistical performance of an estimator can be evaluated with intrinsic Cramér-Rao lower bounds, which generalize the standard Cramér-Rao inequality for parameters that lie in a manifold.

    \item[$\bullet$] 
    Section \ref{sec:ClassifEEG} discusses how the Fisher-Rao geometry of C-CES provides a measure between the distributions of samples that can be leveraged in classification methodologies.
    This framework is then applied to Electroencephalography (EEG) signals.
    
\end{itemize}

\section{An introduction to Riemannian geometry through the Fisher-Rao geometry of CES distributions}
\label{sec:RiemGeo}

This section provides a short introduction to the concepts and tools of Riemannian geometry, while using the Fisher-Rao geometry of C-CES distributions as the main directive example.
Some elementary notions are also assumed to be known for the sake of conciseness (e.g., basics matrix differentiation).
For more detailed coverages of differential geometry, one can refer to the standard textbooks on the topic \cite{gallot1990riemannian,lang2012differential,lee2006riemannian}. 
The notations and definitions of this section are mostly inspired from the books \cite{absil2009optimization,boumal2023introduction}, which provide very good (optimization-oriented) entry points to smooth manifolds and Riemannian geometry. 
The Fisher-Rao geometries of multivariate Gaussian and CES models have been studied in, e.g., \cite{atkinson1981rao,berkane1997geodesic,breloy2018intrinsic, micchelli2005rao,mitchell1989information,smith2005covariance, skovgaard1984riemannian}.
%

\subsection{$\manifold{\nfeatures}$ as a Riemannian manifold}

\label{sec:hpd_as_riem}

The set of $\nfeatures\times\nfeatures$ Hermitian positive definite matrices $\manifold{\nfeatures}$ is an open subspace of the space of $\nfeatures\times\nfeatures$ Hermitian matrices $\ambientSpace{\nfeatures}$.
Since $\manifold{\nfeatures}$ has the same dimension as its embedding space $\ambientSpace{\nfeatures}$, it is a smooth manifold of dimension $\nfeatures^2$~\cite[Definition 3.10]{boumal2023introduction}.
As every smooth manifolds, $\manifold{\nfeatures}$ admits a differential structure, \emph{i.e.}, every point $\point\in\manifold{\nfeatures}$ possesses a tangent space $\tangentSpace{\nfeatures}{\point}$.
The elements of $\tangentSpace{\nfeatures}{\point}$ are called tangent vectors, and correspond to the directional derivatives of curves in $\manifold{\nfeatures}$ passing through $\point$ (cf. Figure~\ref{fig:tangent_from_curves} for an illustration).
Since $\manifold{\nfeatures}$ is open in $\ambientSpace{\nfeatures}$, the tangent space $\tangentSpace{\nfeatures}{\point}$ at every point $\point$ can be identified as $\ambientSpace{\nfeatures}$~\cite[Theorem 3.15]{boumal2023introduction}.
%
%
An illustration of the $1$-dimensional case $\manifold{1}=\realSpace^+_*$ is presented in Figure~\ref{fig:HPD_1_man}.

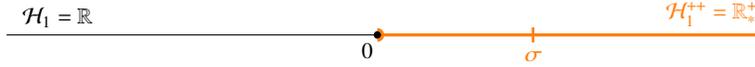
\begin{figure}
    \centering
    \begin{tikzpicture}
    \draw (-5,0) -- (5,0) node[very near start, above, anchor=south east] {$\ambientSpace{1}=\realSpace$};
    
    \draw[draw=myorange,line width=1.2pt] (0,0) -- (5,0) node[very near end, above] {\color{myorange}$\manifold{1}=\realSpace^+_*$};

    \draw[fill=black] (0,0)++(-0.02,0) arc (0:360:0.045) node[anchor=north east] {$0$};

    \draw[draw=myorange, line width=1.2pt] (0,0) arc(0:90:0.07);
    \draw[draw=myorange, line width=1.2pt] (0,0) arc(0:-90:0.07);

    \draw[draw=myorange, line width=1pt] (2,-0.1) -- (2,0.1) node[at start, anchor=north] {\color{myorange}$\sigma$};
\end{tikzpicture}
    \caption{Illustration of the manifold $\manifold{1} = \realSpace^+_*$, which is open in $\ambientSpace{1}=\realSpace$. We observe that the tangent space at every point $\sigma>0$ simply corresponds to $\ambientSpace{1}=\realSpace$.}
    \label{fig:HPD_1_man}
\end{figure}

\noindent
\Remark{
The space $\manifold{\nfeatures}$ is often referred to as the \textit{convex cone} of positive definite matrices.
It is indeed a cone because $\point \in \manifold{\nfeatures}$ implies that
$a \point \in \manifold{\nfeatures},~\forall~a \in \realSpace_+^*$.  It is furthermore a convex cone because any linear combination $a \point_1 + b\point_2 $ is also in $\manifold{\nfeatures}$, $\forall~\point_1, \point_2 \in \manifold{\nfeatures}$ and $\forall~a, b \in \realSpace_+^*$.
This cone visually appears in the real $2\times 2$ case of $\mathcal{S}_2^{++}$, which is often used to represent $\manifold{\nfeatures}$.
Still, this chapter will rely on the representation of Figure \ref{fig:tangent_from_curves},
which is more convenient to illustrate the generic concepts and tools of Riemannian geometry.
}


In order to further harness the differential structure of $\manifold{\nfeatures}$, we endow it with a Riemannian metric.
This consists in a mapping that equips every tangent space $\tangentSpace{\nfeatures}{\point}$ with an inner product\footnote{An inner product is a bilinear, symmetric, positive definite function.} $\metric{\point}{\cdot}{\cdot}$ that varies smoothly with respect to the point $\point$.
This allows notably for locally defining the notion of angle and length for vectors in $\tangentSpace{\nfeatures}{\point}$.
%
%
A smooth manifold equipped with such Riemmanian metric is then referred to as a Riemannian manifold.
Notice that the definition of the metric is a choice that induces a corresponding geometry.
In particular, if $\manifold{\nfeatures}$ is endowed with the Euclidean metric $\metric{\point}{\tangentVector}{\tangentVectorBis}^{\mathcal{E}}=\reel(\tr(\tangentVector\tangentVectorBis))$, where $\reel(\cdot)$ returns the real part of its argument, all the geometrical objects of the manifold $\manifold{\nfeatures}$ are exactly the same as those of the space $\ambientSpace{\nfeatures}$.
In this case, there is no distinction between $\manifold{\nfeatures}$ and $\ambientSpace{\nfeatures}$ from a geometrical point of view, and the true structure of $\manifold{\nfeatures}$ cannot be exploited.
This motivates the use of other metrics, that induce a more meaningful geometry on $\manifold{\nfeatures}$ (e.g., ensuring that the boundaries of the space are not reachable).
In this scope, various options have been considered, such as the affine invariant metric~\cite{bhatia2009positive,moakher2005differential}, the log-Euclidean metric~\cite{arsigny2006log}, or the Bures-Wasserstein one~\cite{bhatia2019bures,han2021riemannian}.
Overviews of the different metrics and their corresponding geometries can be found in \cite{thanwerdas22theseis,thanwerdas2023n}.


        
    
    

%
When dealing with a statistical model the Fisher information metric is generally to be favored, as it is naturally suited to the underlying geometry of the data.
Without resorting to the tedious parameterization and identification of Section \ref{sec:from_fim_to_rg}, a general expression of this metric can directly be obtained following~\cite[Theorem 1]{smith2005covariance} as:
\begin{equation}
    \metric{\point}{\tangentVector}{\tangentVectorBis}^{\textup{FIM}}
     = \expectation{\Diff\loglikelihood(\data|\point)[\tangentVector]\cdot\Diff\loglikelihood(\data|\point)[\tangentVectorBis]}
     = - \expectation{\Diff^2\loglikelihood(\data|\point)[\tangentVector,\tangentVectorBis]},
     \label{eq:fim_with_expectations}
\end{equation}
where $\Diff\loglikelihood$ and $\Diff^2\loglikelihood$ are the first and second order directional derivatives of the log-likelihood $\loglikelihood$ of the distribution with respect to $\point$.
Recall from~\cite{higham2008functions} that the first and second derivatives of a function $L:\manifold{\nfeatures}\to\mathbb{R}$ at $\point\in\manifold{\nfeatures}$ in directions $\tangentVector$ and $\tangentVectorBis\in\tangentSpace{\nfeatures}{\point}$ are defined as
\begin{equation}
    \begin{array}{rcl}
         \Diff L(\point)[\tangentVector] & = & L(\point+\tangentVector) - L(\point) + o(\|\tangentVector\|) \\
         \Diff^2 L(\point)[\tangentVector,\tangentVectorBis] & = & \Diff L(\point+\tangentVectorBis)[\tangentVector] - \Diff L(\point)[\tangentVector] + o(\|\tangentVector\|).
    \end{array}
\end{equation}
Notice that $\Diff^2 L(\point)[\tangentVector,\tangentVectorBis]$ is symmetrical with respect to $\tangentVector$ and $\tangentVectorBis$.
In the case of the CES distributions, the Fisher information metric was studied in \cite{atkinson1981rao,berkane1997geodesic,breloy2018intrinsic, micchelli2005rao,mitchell1989information}, and its derivation is reported in the following Theorem:

%
\begin{theorem}[Fisher Information metric of centered CES]
    Let $\point \in \manifold{\nfeatures}$.
    Let $\{\data_i\}_{i=1}^{\nsamples}$ be a $\nsamples$-sample of iid from $\data \sim \distribution{\point}{\densityGenerator}$.
    The Fisher information metric is obtained $\forall~\tangentVector,\tangentVectorBis\in\tangentSpace{\nfeatures}{\point}$ as
    \begin{equation*}
    \metric{\point}{\tangentVector}{\tangentVectorBis}^{\textup{FIM}}
    =
    n \alpha_{\densityGenerator} \tr(\point^{-1}\tangentVector\point^{-1}\tangentVectorBis) 
    +
    n \beta_g \tr(\point^{-1}\tangentVector)\tr(\point^{-1}\tangentVectorBis),
\end{equation*}
with $\alpha_{\densityGenerator}$ and $\beta_{\densityGenerator}$ defined in \eqref{eq:coefficients_metric}.
\label{thm:fim_metric}
\end{theorem}
\begin{proof}
    The first things to compute to obtain the Fisher information metric are the derivatives $\Diff\loglikelihood(\{\data_i\}_{i=1}^{\nsamples}|\point)[\tangentVector]$ and $\Diff^2\loglikelihood(\{\data_i\}_{i=1}^{\nsamples}|\point)[\tangentVector,\tangentVectorBis]$ at $\point\in\manifold{\nfeatures}$ in directions $\tangentVector$ and $\tangentVectorBis\in\tangentSpace{\nfeatures}{\point}$.
    To do so, recall that $\Diff\log\det(\point)[\tangentVector]=\tr(\point^{-1}\tangentVector)$ and $\Diff(\point^{-1})[\tangentVector]=-\point^{-1}\tangentVector\point^{-1}$.
    It follows that
    \begin{equation*}
        \Diff\loglikelihood(\{\data_i\}_{i=1}^{\nsamples}|\point)[\tangentVector]
        = -\nsamples\tr(\point^{-1}\tangentVector) - \sum_{i=1}^{\nsamples} \phi(\data_i^H\point^{-1}\data_i) \tr(\point^{-1}\tangentVector\point^{-1}\data_i\data_i^H).
    \end{equation*}
    Moreover, also recall that the trace is invariant to any permutation of the product of three Hermitian matrices.
    Thus,
    \begin{multline*}
        \Diff^2\loglikelihood(\{\data_i\}_{i=1}^{\nsamples}|\point)[\tangentVector,\tangentVectorBis]
        =
        \tr(\point^{-1}\tangentVector\point^{-1}\tangentVectorBis)
        \\
        + 2\sum_{i=1}^{\nsamples} \phi(\tr(\point^{-1}\data_i\data_i^H)) \tr(\point^{-1}\tangentVector\point^{-1}\tangentVectorBis\point^{-1}\data_i\data_i^H)
        \\
        + \sum_{i=1}^{\nsamples} \phi'(\tr(\point^{-1}\data_i\data_i^H)) \tr(\point^{-1}\tangentVector\point^{-1}\data_i\data_i^H) \tr(\point^{-1}\tangentVectorBis\point^{-1}\data_i\data_i^H).
    \end{multline*}
    We now need to compute the expectation.
    To do so, we exploit the stochastic representation $\data_i {=}_d \sqrt{\mathcal{Q}_i}\point^{\nicefrac12}\mathbf{u}_i$.
    Recall that $\mathcal{Q}_i$ and $\mathbf{u}_i$ are independent, $\mathbf{u}_i^H\mathbf{u}_i=1$ and $\expectation{\mathbf{u}_i\mathbf{u}_i^H}=\frac1p\mathbf{I}_p$ (since $\mathbf{u}_i\sim\mathcal{U}(\complexSpace\mathcal{S}^\nfeatures)$).
    Furthermore, from~\eqref{eq:pdf_2modular_variate}, $\expectation{\mathcal{Q}_i\phi(\mathcal{Q}_i)}=-\nfeatures$.
    It follows that
    \begin{equation*}
    \begin{aligned}
        &\expectation{\phi(\tr(\point^{-1}\data_i\data_i^H)) \tr(\point^{-1}\tangentVector\point^{-1}\tangentVectorBis\point^{-1}\data_i\data_i^H) }
        \\
        &\quad=
        \tr(\point^{-1}\tangentVector\point^{-1}\tangentVectorBis\expectation{\phi(\mathcal{Q}_i)\mathcal{Q}_i\mathbf{u}_i\mathbf{u}_i^H})
        = \tr(\point^{-1}\tangentVector\point^{-1}\tangentVectorBis\expectation{\phi(\mathcal{Q}_i)\mathcal{Q}_i}\expectation{\mathbf{u}_i\mathbf{u}_i^H})
        \\
        &\quad= -\tr(\point^{-1}\tangentVector\point^{-1}\tangentVectorBis).
    \end{aligned}
    \end{equation*}
    For the second expectation, from~\cite{besson2013fisher}, we need
    \begin{equation*}
        \expectation{(\mathbf{u}_i^H\MAT{A}\mathbf{u}_i)^2}
        = \frac{\tr(\MAT{A}^2)+(\tr(\MAT{A}))^2}{p(p+1)}.
    \end{equation*}
    Applying the polarization formula $\frac14[(\mathbf{u}_i^H(\MAT{A}+\MAT{B})\mathbf{u}_i)^2-(\mathbf{u}_i^H(\MAT{A}-\MAT{B})\mathbf{u}_i)^2]$, we get
    \begin{equation*}
        \expectation{(\mathbf{u}_i^H\MAT{A}\mathbf{u}_i)(\mathbf{u}_i^H\MAT{B}\mathbf{u}_i)}
        = \frac{\tr(\MAT{A}\MAT{B})+\tr(\MAT{A})\tr(\MAT{B})}{p(p+1)}.
    \end{equation*}
    Therefore,
    \begin{equation*}
        \begin{aligned}
            & \expectation{\phi'(\tr(\point^{-1}\data_i\data_i^H))\tr(\point^{-1}\tangentVector\point^{-1}\data_i\data_i^H) \tr(\point^{-1}\tangentVectorBis\point^{-1}\data_i\data_i^H)}
            \\
            &\quad = \expectation{\mathcal{Q}_i^2\phi'(\mathcal{Q}_i)(\mathbf{u}_i^H\point^{-1}\tangentVector\mathbf{u}_i)(\mathbf{u}_i^H\point^{-1}\tangentVectorBis\mathbf{u}_i)}
            \\
            &\quad = \expectation{\mathcal{Q}_i^2\phi'(\mathcal{Q}_i)}\expectation{(\mathbf{u}_i^H\point^{-1}\tangentVector\mathbf{u}_i)(\mathbf{u}_i^H\point^{-1}\tangentVectorBis\mathbf{u}_i)}
            \\
            &\quad = \frac{\expectation{\mathcal{Q}_i^2\phi'(\mathcal{Q}_i)}}{p(p+1)}\left( \tr(\point^{-1}\tangentVector\point^{-1}\tangentVectorBis) + \tr(\point^{-1}\tangentVector)\tr(\point^{-1}\tangentVectorBis) \right)
        \end{aligned}
    \end{equation*}
    From there, basic manipulations yield the result with coefficients $\alpha_g$ and $\beta_g$ defined in \eqref{eq:coefficients_metric}.
    Notice that the dependency on $i$ in $\mathcal{Q}$ is omitted since these parameters are assumed iid.
\end{proof}

\noindent
The Fisher information metric of C-CES thus corresponds to a general form of the well known affine invariant metric on $\manifold{\nfeatures}$~\cite{bhatia2009positive,skovgaard1984riemannian}.
Hence, if not specified otherwise the remainder of this chapter will use the more common generic denotation:
\begin{equation} \label{eq:AI_metric}
    \metric{\point}{\tangentVector}{\tangentVectorBis}
    \overset{\text{def}}{=}
    \mathfrak{g}_{\point}(\tangentVector,\tangentVectorBis)
    =
    \metricCES{\point}{\tangentVector}{\tangentVectorBis}
\end{equation}
and study the corresponding Riemannian geometry of $\manifold{\nfeatures}$ for any $\alpha\in\realSpace^+_*$ and $\beta > - \alpha / p $ (necessary conditions so that $\metric{\point}{\cdot}{\cdot}$ is positive definite).
The Fisher-Rao information geometry of the considered C-CES model is then recovered by fixing $\alpha$ and $\beta$ according to Theorem~\ref{thm:fim_metric}.
We can also point out that the most studied case corresponds to $\alpha=1$ and $\beta=0$, which coincide with the Fisher information metric of the Gaussian distribution, as $\alpha_g = 1$ and $\beta_g=0 $ in this case~\cite{breloy2018intrinsic,smith2005covariance}. 

\Remark{
    Taking the real part in the metric~\eqref{eq:AI_metric} defines a proper inner product on $\tangentSpace{\nfeatures}{\point}$ from the original Hermitian inner product.
    This way, we implicitly identify the complex space as its underlying real vector space ($\complexSpace\sim \realSpace^2$), so that we can use the usual derivatives (defined as those used on $\realSpace$).
    As a direct consequence, in this chapter, both $\ambientSpace{\nfeatures}$ and $\manifold{\nfeatures}$ are of dimension $\nfeatures^2$.
    Notice that, even though it is not always stated, most works that deal with complex-valued matrices (e.g., \cite{smith2005covariance}) also implicitly use the real part of the Fisher information metric.
}

\Remark{
Among many other properties, the Fisher information metric from Theorem~\ref{thm:fim_metric} has a notable quadratic dependence on $\point^{-1}$.
This makes the norm of tangent vectors $\|\tangentVector\|^2_{\point}=\metric{\point}{\tangentVector}{\tangentVector}$ tend to infinity when the point $\point$ tends to the boundaries of the manifold (i.e, when any number of its eigenvalues tend $0$).
This Riemannian metric thus allows to actually perceive the boundary of $\manifold{\nfeatures}$ as being infinitely far, which was not the case for the Euclidean metric.
An illustration of the effect of the metric is displayed for $\manifold{1}=\realSpace^+_*$ in Figure~\ref{fig:HPD_1_metric}.
}

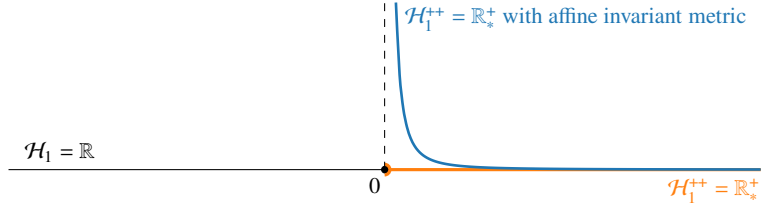
\begin{figure}[!h]
    \centering
    \begin{tikzpicture}
    \draw (-5,0) -- (5,0) node[very near start, above, anchor=south east] {$\ambientSpace{1}=\realSpace$};
    
    \draw[draw=myorange,line width=1.2pt] (0,0) -- (5,0) node[very near end, below] {\color{myorange}$\manifold{1}=\realSpace^+_*$};

    \draw[fill=black] (0.05,0) arc (0:360:0.045) node[anchor=north east] {$0$};

    \draw[draw=myorange, line width=1.2pt] (0.07,0) arc(0:90:0.07);
    \draw[draw=myorange, line width=1.2pt] (0.07,0) arc(0:-90:0.07);

    \draw[dashed] (0,0) -- (0,2.3);
    
    \draw[draw=myblue,line width=1pt,smooth,samples=100,domain=0.15:5] plot(\x,{(0.05/\x^2)});

    \node[anchor=west] at (0.15,2) {\color{myblue}$\manifold{1}=\realSpace^+_*$ with affine invariant metric};
\end{tikzpicture}
    \caption{Illustration of the effect of the affine invariant metric on $\manifold{1} = \realSpace^+_*$. Thanks to the metric, the excluded point $0$ becomes truly unreachable.}
    \label{fig:HPD_1_metric}
\end{figure}

\subsection{Levi-Civita connection}

One of the most -- if not the most -- important tools of Riemannian geometry is the Levi-Civita connection, which generalizes the notion of directional derivatives of vector fields on manifolds.
%
A vector field is a function that associates a unique tangent vector $\tangentVector_{\point} \in \tangentSpace{\nfeatures}{\point}$ to every point $\point\in\manifold{\nfeatures}$, which is illustrated in Figure~\ref{fig:vector_field}.
An example of a vector field that will be involved in Section~\ref{sec:Optim} is the gradient of a cost function.
The set of vector fields on $\manifold{\nfeatures}$ is denoted $\mathfrak{X}(\manifold{\nfeatures})$.
%

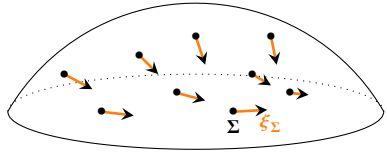
\begin{figure}
    \centering
    \begin{tikzpicture}[scale=2.5]
\draw  (1,0) arc (30:150:1.155)
      plot [smooth, domain=pi:2*pi] ({cos(\x r)},{0.2*sin(\x r)});
\draw [dotted] plot [smooth, domain=0:pi] ({cos(\x r)},{0.2*sin(\x r)});

\draw[->,>=stealth,draw=myorange,line width=1pt] (-0.7,0.2) -- ++(0.15,-0.08) node[at start] {\scriptsize$\bullet$};

\draw[->,>=stealth,draw=myorange,line width=1pt] (-0.5,0) -- ++(0.17,-0.02) node[at start] {\scriptsize$\bullet$};

\draw[->,>=stealth,draw=myorange,line width=1pt] (-0.3,0.3) -- ++(0.1,-0.1) node[at start] {\scriptsize$\bullet$};

\draw[->,>=stealth,draw=myorange,line width=1pt] (-0.1,0.1) -- ++(0.15,-0.04) node[at start] {\scriptsize$\bullet$};

\draw[->,>=stealth,draw=myorange,line width=1pt] (0,0.4) -- ++(0.05,-0.15) node[at start] {\scriptsize$\bullet$};

\draw[->,>=stealth,draw=myorange,line width=1pt] (0.2,0) -- ++(0.18,0.01) node[at start] {\scriptsize$\bullet$};

\draw[->,>=stealth,draw=myorange,line width=1pt] (0.3,0.2) -- ++(0.1,-0.06) node[at start] {\scriptsize$\bullet$};

\draw[->,>=stealth,draw=myorange,line width=1pt] (0.4,0.4) -- ++(0.03,-0.15) node[at start] {\scriptsize$\bullet$};

\draw[->,>=stealth,draw=myorange,line width=1pt] (0.5,0.1) -- ++(0.1,-0.01) node[at start] {\scriptsize$\bullet$};

\node[anchor=north] at (0.2,0) {\scriptsize$\point$};
\node[anchor=north] at (0.4,0.03) {\color{myorange}\scriptsize$\tangentVector_{\point}$};

\end{tikzpicture}
    \caption{Illustration of a vector field on $\manifold{\nfeatures}$.}
    \label{fig:vector_field}
\end{figure}

To differentiate a vector field on a manifold, one needs to resort to an affine connection.
This is an application from $\mathfrak{X}(\manifold{\nfeatures})\times \mathfrak{X}(\manifold{\nfeatures})$ onto $\mathfrak{X}(\manifold{\nfeatures})$.
The connection of $\tangentVectorBis_{\point}$ in the direction $\tangentVector_{\point}$ is denoted $\nabla_{\tangentVector_{\point}} \, \tangentVectorBis_{\point}$ and generalizes the directional derivative of $\tangentVectorBis_{\point}$ in the direction $\tangentVector_{\point}$ (i.e., $\Diff\tangentVectorBis_{\point}[\tangentVector_{\point}]$).
Such generalization is needed because the tangent space changes when one moves from one point to another on a manifold.
Thus, the usual directional derivative might not be properly defined, as it does not account for the structure of the manifold (constraints, Riemannian metric, etc.).
This specificity is illustrated in Figure~\ref{fig:connection}.
\begin{figure}[!h]
    \centering
    \begin{center}
\scriptsize
\begin{tikzpicture}[scale=1.4]

\draw[line width=0.7pt] (0,0) -- (2,0);

\draw[color=myblue, line width=1pt, ->, >=stealth] (0,0) -- ++(0.8,0) node[near end, below] {$\tangentVector_{\point}$};
\draw[color=myorange, line width=1pt, ->, >=stealth] (0,0) -- ++(0.4,0.7) node[very near end, above, anchor=south east] {$\tangentVectorBis_{\point}$};
\draw[color=myred, line width=1pt, ->, >=stealth] (0,0) -- ++(0.3,-0.3) node[at end, below, anchor=north] {$\Diff\tangentVectorBis_{\point}[\tangentVector_{\point}]$};
\draw[color=myorange, dashed, line width=0.7pt, ->, >=stealth] (0,0) -- ++(0.7,0.4);
\draw[color=myred, dashed, line width=0.7pt, ->, >=stealth] (0.4,0.7) -- ++(0.3,-0.3);
\node at (0,0) {\tiny$\bullet$};
\node[anchor=north east] at (0,0) {$\point$};

\draw[color=myblue, line width=1pt, ->, >=stealth] (2,0) -- ++(0.8,0) node[near end, below] {$\tangentVector_{\pointBis}$};
\draw[color=myorange, dashed, line width=0.7pt, ->, >=stealth] (2,0) -- ++(0.4,0.7);
\draw[color=myorange, line width=1pt, ->, >=stealth] (2,0) -- ++(0.7,0.4) node[at end, below, anchor=west] {$\tangentVectorBis_{\pointBis}$};
\node at (2,0) {\tiny$\bullet$};
\node[anchor=north east] at (2,0) {$\pointBis$};

\begin{scope}[shift={(5.5,-0.4)},scale=2.1]
    \draw  (1,0) arc (30:150:1.155)
      plot [smooth, domain=pi:2*pi] ({cos(\x r)},{0.2*sin(\x r)});
    \draw [dotted] plot [smooth, domain=0:pi] ({cos(\x r)},{0.2*sin(\x r)});
\end{scope}

\begin{scope}[shift={(4.5,0)}]
    
    \draw[line width=0.7pt] (0,-0.2) to[bend right=17] (1.6,-0.25);

    \draw[color=myblue, line width=1pt, ->, >=stealth] (0,-0.2) -- ++(0.5,-0.2) node[near end, below, anchor=north west] {$\tangentVector_{\point}$};
    \draw[color=myorange, line width=1pt, ->, >=stealth] (0,-0.2) -- ++(0.27,0.45) node[near end, above, anchor=east] {$\tangentVectorBis_{\point}$};
    \draw[color=myorange, dashed, line width=0.7pt, ->, >=stealth] (0,-0.2) -- ++(0.47,0.19);
    \draw[color=myred, line width=1pt, ->, >=stealth] (0,-0.2) -- ++(0.2,-0.26) node[near end, below, anchor=north east] {$\nabla_{\tangentVector_{\point}}\tangentVectorBis_{\point}$};
    \draw[color=myred, dashed, line width=0.7pt, ->, >=stealth] (0.27,0.25) -- ++(0.2,-0.26);
    \node at (0,-0.2) {\tiny$\bullet$};
    \node[anchor=east] at (0,-0.2) {$\point$};

    \draw[color=myblue, line width=1pt, ->, >=stealth] (1.6,-0.25) -- ++(0.5,0.17) node[midway, below, anchor=north west] {$\tangentVector_{\pointBis}$};

    \draw[color=myorange, dashed, line width=0.7pt, ->, >=stealth] (1.6,-0.25) -- ++(-0.1,0.52);

    \draw[color=myorange, line width=1pt, ->, >=stealth] (1.6,-0.25) -- ++(0.22,0.4) node[at end, below, anchor=west] {$\tangentVectorBis_{\pointBis}$};
    
    \node at (1.6,-0.25) {\tiny$\bullet$};
    \node[anchor=north east] at (1.6,-0.25) {$\pointBis$};
\end{scope}

\end{tikzpicture}
\end{center}
    \caption{
    Illustration of directional derivative $\Diff\tangentVectorBis_{\point}[\tangentVector_{\point}]$ (left) and affine connection $\nabla_{\tangentVector_{\point}} \, \tangentVectorBis_{\point}$ (right) of a vector field $\tangentVectorBis$ in the direction $\tangentVector$ at $\point$.
    As the directional derivative, an affine connection describes how the vector field $\tangentVectorBis$ evolves in a given direction $\tangentVector$.
    In addition, the affine connection takes into account the structure of the manifold (curvature, and non-constant metric).
    }
    \label{fig:connection}
\end{figure}
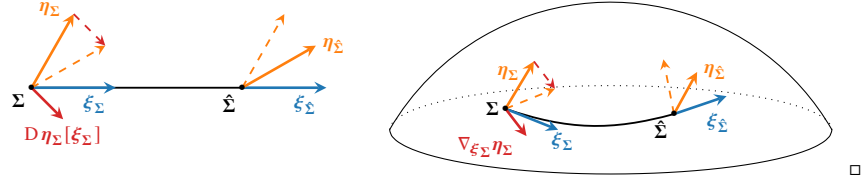

\noindent
Many affine connections can be defined on a manifold.
However, there is a unique one that is in accordance with the chosen Riemannian metric, which is referred to as the Levi-Civita connection.
%
%
This Levi-Civita connection, denoted $\nabla_{\tangentVector_{\point}} \, \tangentVectorBis_{\point}$, is the unique solution in the tangent space $\tangentSpace{\nfeatures}{\point}$ to the Koszul formula
\begin{multline}\label{eq:Koszul}
    2 \mathfrak{g}_{\point}(\nabla_{\tangentVector_{\point}} \, \tangentVectorBis_{\point}, \tangentVectorTer_{\point}) = 
    2 \mathfrak{g}_{\point}(\Diff\tangentVectorBis_{\point}[\tangentVector_{\point}], \tangentVectorTer_{\point})
    \\
    + \Diff \mathfrak{g}_{\point}[\tangentVector_{\point}](\tangentVectorBis_{\point},\tangentVectorTer_{\point})
    + \Diff \mathfrak{g}_{\point}[\tangentVectorBis_{\point}](\tangentVector_{\point},\tangentVectorTer_{\point})
    - \Diff \mathfrak{g}_{\point}[\tangentVectorTer_{\point}](\tangentVectorBis_{\point},\tangentVector_{\point}),
\end{multline}
where we use the alternate notation of the metric, i.e., $\mathfrak{g}_{\point}(\cdot,\cdot)=\metric{\point}{\cdot}{\cdot}$.
Notice that the presented formula is simpler than the general case~\cite{absil2009optimization}.
It is because the Lie bracket is $[\tangentVector_{\point},\tangentVectorBis_{\point}] = \Diff\tangentVectorBis_{\point}[\tangentVectorBis_{\point}] - \Diff\tangentVector_{\point}[\tangentVectorBis_{\point}]$ since $\manifold{\nfeatures}$ is an open subset of a vector space, i.e., $\ambientSpace{\nfeatures}$.
%
%
%
%
%
The Levi-Civita connection of $\manifold{\nfeatures}$ associated with the Riemannian metric~\eqref{eq:AI_metric} is provided in Theorem~\ref{thm:connection}.

\begin{theorem}[Levi-Civita connection]
    The Levi-Civita connection on $\manifold{\nfeatures}$ associated with the affine invariant metric~\eqref{eq:AI_metric} is defined for $\tangentVector$, $\tangentVectorBis\in\mathfrak{X}(\manifold{\nfeatures})$ and $\point\in\manifold{\nfeatures}$, as
    \begin{equation*}
        \nabla_{\tangentVector_{\point}} \, \tangentVectorBis_{\point} = \Diff\tangentVectorBis_{\point}[\tangentVector_{\point}] - \symm(\tangentVectorBis_{\point}\point^{-1}\tangentVector_{\point}),
    \end{equation*}
    where $\symm(\cdot)$ returns the Hermitian part of its argument.
\label{thm:connection}
\end{theorem}
\begin{proof}
    First recall that for $\MAT{A}\in\ambientSpace{\nfeatures}$ and $\MAT{B}\in\realSpace^{p\times p}$, $\tr(\MAT{A}\MAT{B}) = \tr(\MAT{A}\symm(\MAT{B}))$.
    Further recall that the trace is invariant to any permutation of the product of three Hermitian matrices.
    Since $\Diff(\point^{-1})[\tangentVector]=-\point^{-1}\tangentVector\point^{-1}$, we have
    \begin{equation*}
        \Diff\reel(\tr(\point^{-1}\tangentVector\point^{-1}\tangentVectorBis))[\tangentVectorTer]
        =
        -2\reel(\tr(\point^{-1}\tangentVector\point^{-1}\tangentVectorBis\point^{-1}\tangentVectorTer))
    \end{equation*}
    and
    \begin{multline*}
        \Diff\reel(\tr(\point^{-1}\tangentVectorBis\point^{-1}\tangentVectorTer))[\tangentVector]
        + \Diff\reel(\tr(\point^{-1}\tangentVector\point^{-1}\tangentVectorTer))[\tangentVectorBis]
        - \Diff\reel\tr(\point^{-1}\tangentVector\point^{-1}\tangentVectorBis))[\tangentVectorTer]
        \\
        =
        -2\reel(\tr(\point^{-1}\tangentVector\point^{-1}\tangentVectorBis\point^{-1}\tangentVectorTer))
        =
        -2\reel(\tr(\point^{-1}\symm(\tangentVector\point^{-1}\tangentVectorBis)\point^{-1}\tangentVectorTer)).
    \end{multline*}
    We also have
    \begin{multline*}
        \Diff\left(\reel(\tr(\point^{-1}\tangentVector)\tr(\point^{-1}\tangentVectorBis))\right)[\tangentVectorTer]
        =
        - \reel(\tr(\point^{-1}\tangentVectorTer\point^{-1}\tangentVector)\tr(\point^{-1}\tangentVectorBis))
        \\
        - \reel(\tr(\point^{-1}\tangentVector)\tr(\point^{-1}\tangentVectorTer\point^{-1}\tangentVectorBis)).
    \end{multline*}
    It follows that
    \begin{multline*}
        \Diff\left(\reel(\tr(\point^{-1}\tangentVectorBis)\tr(\point^{-1}\tangentVectorTer))\right)[\tangentVector]
        + \Diff\left(\reel(\tr(\point^{-1}\tangentVector)\tr(\point^{-1}\tangentVectorTer))\right)[\tangentVectorBis]
        \\
        - \Diff\left(\reel(\tr(\point^{-1}\tangentVector)\tr(\point^{-1}\tangentVectorBis))\right)[\tangentVectorTer]
        = -2\reel(\tr(\point^{-1}\tangentVectorTer)\tr(\point^{-1}\tangentVector\point^{-1}\tangentVectorBis)) \hspace{7pt}
        \\
        = -2\reel(\tr(\point^{-1}\tangentVectorTer)\tr(\point^{-1}\symm(\tangentVector\point^{-1}\tangentVectorBis))).
    \end{multline*}
    From there, we can deduce that
    \begin{multline*}
        \Diff \mathfrak{g}_{\point}[\tangentVector_{\point}](\tangentVectorBis_{\point},\tangentVectorTer_{\point})
        + \Diff \mathfrak{g}_{\point}[\tangentVectorBis_{\point}](\tangentVector_{\point},\tangentVectorTer_{\point})
        - \Diff \mathfrak{g}_{\point}[\tangentVectorTer_{\point}](\tangentVectorBis_{\point},\tangentVector_{\point})
        \\
        = -2\mathfrak{g}_{\point}(\symm(\tangentVector_{\point}\point^{-1}\tangentVectorBis_{\point}), \tangentVectorTer_{\point}).
    \end{multline*}
    Injecting this into the Koszul formula yields the result.
\end{proof}

\Remark{
Notice that the Levi-Civita connection of $\manifold{\nfeatures}$ associated with the Riemannian metric of the metric in \eqref{eq:AI_metric} does not depend on $\alpha$ and $\beta$.
Hence it remains the same for any underlying C-CES distribution.
}

\subsection{Geodesics, Riemannian exponential, logarithm and distance}

\label{sec:geodesicx_exp_log_dist}

One of the main reasons why the Levi-Civita connection is so crucial is because it allows to define geodesics.
The geodesics generalize the concept of straight lines on a manifold.
These are curves $\gamma:[0,1]\to\manifold{\nfeatures}$ with no acceleration, where acceleration is defined thanks to the Levi-Civita connection.
They are parameterized by the choice of starting point $\gamma(0)=\point\in\manifold{\nfeatures}$ and either initial direction $\dot{\gamma}(0)=\tangentVector\in\ambientSpace{\nfeatures}$ or ending point $\gamma(1)=\pointBis\in\manifold{\nfeatures}$.
An illustration of geodesics is provided in Figure~\ref{fig:geodesic_exp_log}.
Formally, the geodesic $\gamma:[0,1]\to\manifold{\nfeatures}$ is the solution to the differential equation
\begin{equation}
    \nabla_{\dot{\gamma}(t)} \, \dot{\gamma}(t) = \MAT{0}.
\label{eq:geodesics_def}
\end{equation}
The geodesics on $\manifold{\nfeatures}$ according to the Levi-Civita connection of Theorem \ref{thm:connection} are given in Theorem~\ref{thm:geodesics} along with the proof.

\begin{theorem}[Geodesics]
    The geodesic $\gamma:[0,1]\to\manifold{\nfeatures}$ such that $\gamma(0)=\point\in\manifold{\nfeatures}$ and $\dot{\gamma}(0)=\tangentVector\in\ambientSpace{\nfeatures}$ is defined as
    \begin{equation*}
        \gamma(t)
        \; = \; \point\exp(t\point^{-1}\tangentVector)
        \; = \; \exp(t\tangentVector\point^{-1})\point
        \; = \; \point^{\nicefrac12}\exp(t\point^{-\nicefrac12}\tangentVector\point^{-\nicefrac12})\point^{\nicefrac12},
    \end{equation*}
    where $\exp(\cdot)$ denotes the matrix exponential.
    Equivalently, one can define the geodesic $\gamma:[0,1]\to\manifold{\nfeatures}$ such that $\gamma(0)=\point\in\manifold{\nfeatures}$ and $\gamma(1)=\pointBis\in\manifold{\nfeatures}$~by
    \begin{equation*}
        \gamma(t) = \point^{\nicefrac12}\left(\point^{-\nicefrac12}\pointBis\point^{-\nicefrac12}\right)^t\point^{\nicefrac12},
    \end{equation*}
    where $(\cdot)^t=\exp(t\log(\cdot))$, $\log(\cdot)$ denoting the matrix logarithm.
\label{thm:geodesics}
\end{theorem}
\begin{proof}
    We only provide the proof for $\gamma(0)=\point$ and $\dot{\gamma}(0)=\tangentVector$.
    The result for $\gamma(1)=\pointBis$ is obtained by choosing
    \begin{equation*}
        \tangentVector 
        \; = \; \point\log(\point^{-1}\pointBis)
        \; = \; \log(\pointBis\point^{-1})\point
        \; = \; \point^{\nicefrac12}\log(\point^{-\nicefrac12}\pointBis\point^{-\nicefrac12})\point^{\nicefrac12}.
    \end{equation*}
    Notice that the equality between the three versions of $\gamma(t)$ given in the theorem above and of $\tangentVector$ given here solely rely on the fact that $\MAT{A}\exp(\MAT{B})\MAT{A}^{-1} = \exp(\MAT{A}\MAT{B}\MAT{A}^{-1})$ and $\MAT{A}\log(\MAT{B})\MAT{A}^{-1} = \log(\MAT{A}\MAT{B}\MAT{A}^{-1})$.

    The differential equation~\eqref{eq:geodesics_def} for the Levi-Civita connection defined in Theorem~\ref{thm:connection}~is
    \begin{equation*}
        \Ddot{\gamma}(t) - \dot{\gamma}(t)\gamma(t)^{-1}\dot{\gamma}(t) = \MAT{0}.
    \end{equation*}
    Recall that $\frac{\D}{\D t}\exp(t\MAT{A})=\MAT{A}\exp(t\MAT{A})$.
    Thus, with $\gamma(t)=\point\exp(t\point^{-1}\tangentVector)$, we have $\dot{\gamma}(t) = \tangentVector\exp(t\point^{-1}\tangentVector)$ and $\Ddot{\gamma}(t)=\tangentVector\point^{-1}\tangentVector\exp(t\point^{-1}\tangentVector)$.
    From there, we easily obtain $\dot{\gamma}(t)\gamma(t)^{-1}=\tangentVector\point^{-1}$.
    Simple computations show that $\gamma(t)$ satisfies the differential equation above, which is enough to conclude.
\end{proof}

\Remark{Since the Levi-Civita connection does not depend on $\alpha$ and $\beta$, neither does geodesics.
Hence, the Fisher-Rao geometries of C-CES models share the same geodesics whatever the underlying distribution.
}

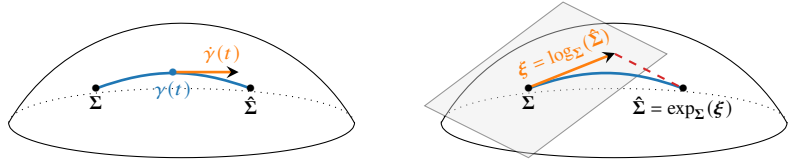
\begin{figure}
    \centering
    \begin{tikzpicture}[scale=2.3]
\draw  (1,0) arc (30:150:1.155)
      plot [smooth, domain=pi:2*pi] ({cos(\x r)},{0.2*sin(\x r)});
\draw [dotted] plot [smooth, domain=0:pi] ({cos(\x r)},{0.2*sin(\x r)});

\coordinate (x) at (-0.495,0.2);
\coordinate (y) at (0.4,0.2);

\draw[draw=myblue, line width=1pt] (x) to[bend left = 20] (y);
\node at (x) {\scriptsize$\bullet$};
\node[anchor=north] at (x) {\scriptsize$\point$};
\node at (y) {\scriptsize$\bullet$};
\node[anchor=north] at (y) {\scriptsize$\pointBis$};

\coordinate (z) at (-0.05,0.29);

\draw[draw=myorange, line width=1pt,->,>=stealth] (z)++(0,0.005) -- ++(0.38,0) node[near end, above] {\scriptsize\color{myorange}$\dot{\gamma}(t)$};

\node at (z) {\scriptsize\color{myblue}$\bullet$};
\node[anchor=north] at (z) {\scriptsize\color{myblue}$\gamma(t)$};

\begin{scope}[shift={(2.5,0)}]
    \draw  (1,0) arc (30:150:1.155)
          plot [smooth, domain=pi:2*pi] ({cos(\x r)},{0.2*sin(\x r)});
    \draw [dotted] plot [smooth, domain=0:pi] ({cos(\x r)},{0.2*sin(\x r)});

    \coordinate (x) at (-0.495,0.2);
    \coordinate (y) at (0.4,0.2);
    
    \draw [fill=gray!20,opacity=0.4] (x)++(0.8,0.2) -- ++(-0.8,-0.6) -- ++(-0.6,+0.3) -- ++(+0.8,+0.6) -- cycle;

    \draw[draw=myblue,line width=1pt] (x) to[bend left = 20] (y);

    \draw[draw=myred,dashed,line width=0.8pt] (x)++(0.5,0.2) -- (y);

    \draw[draw=myorange, line width=1pt,->,>=stealth] (x)++(0,0.005) -- ++(0.5,0.2) node[midway, above, sloped] {\scriptsize\color{myorange}$\MAT{\xi}=\log_{\point}(\pointBis)$};
    
    \node at (x) {\scriptsize$\bullet$};
    \node[anchor=north] at (x) {\scriptsize$\point$};
    \node at (y) {\scriptsize$\bullet$};
    \node[anchor=north] at (y) {\scriptsize$\pointBis=\exp_{\point}(\MAT{\xi})$};
\end{scope}

\end{tikzpicture}
    \caption{Illustration of geodesics (left), Riemannian exponential and logarithm mappings (right).
    The Riemannian distance $\delta(\point,\pointBis)$ is the length of the geodesic joining $\point$ and $\pointBis$.}
    \label{fig:geodesic_exp_log}
\end{figure}

Geodesics allow to define the Riemannian exponential mapping.
By definition, for all $\point\in\manifold{\nfeatures}$, this is the mapping from $\tangentSpace{\nfeatures}{\point}\simeq\ambientSpace{\nfeatures}$ onto $\manifold{\nfeatures}$ such that, for all $\tangentVector\in\ambientSpace{\nfeatures}$, $\exp_{\point}(\tangentVector)=\gamma(1)$, where $\gamma$ is the geodesic such that $\gamma(0)=\point$ and $\dot{\gamma}(t)=\tangentVector$.
Thus, for all $\point\in\manifold{\nfeatures}$ and $\tangentVector\in\ambientSpace{\nfeatures}$, we have
\begin{equation}
    \exp_{\point}(\tangentVector)
    \; = \; \point\exp(\point^{-1}\tangentVector)
    \; = \; \exp(\tangentVector\point^{-1})\point
    \; = \; \point^{\nicefrac12}\exp(\point^{-\nicefrac12}\tangentVector\point^{-\nicefrac12})\point^{\nicefrac12}.
\label{eq:Rexp}
\end{equation}
From there we can define the Riemannian logarithm mapping, which is the inverse of the Riemannian exponential mapping.
Given $\point\in\manifold{\nfeatures}$, it is the mapping from $\manifold{\nfeatures}$ onto $\tangentSpace{\nfeatures}{\point}\simeq\ambientSpace{\nfeatures}$ such that, for $\pointBis\in\manifold{\nfeatures}$, $\log_{\point}(\pointBis)$ is the solution to equation $\exp_{\point}(\log_{\point}(\pointBis))=\pointBis$.
In our case, we have
\begin{equation}
    \log_{\point}(\pointBis)
    \; = \; \point\log(\point^{-1}\pointBis)
    \; = \; \log(\pointBis\point^{-1})\point
    \; = \; \point^{\nicefrac12}\log(\point^{-\nicefrac12}\pointBis\point^{-\nicefrac12})\point^{\nicefrac12}.
\end{equation}
Illustrations of Riemannian exponential and logarithm mappings are given in Figure~\ref{fig:geodesic_exp_log}.

The last object from Riemannian geometry presented in this chapter is the Riemannian distance.
The distance between two points corresponds to the length of the geodesic joining them.
Formally, it is defined as
\begin{equation}
    \delta(\point,\pointBis) = \int_0^1 \langle\dot{\gamma}(t),\dot{\gamma}(t)\rangle_{\gamma(t)}^{\nicefrac12} \D t,
\end{equation}
where $\gamma$ is the geodesic such that $\gamma(0)=\point$ and $\gamma(1)=\pointBis$.
The Riemannian distance on $\manifold{\nfeatures}$ associated to the metric of Theorem~\ref{thm:fim_metric} is given in Theorem~\ref{thm:distance_ces} along with the proof.
It was derived in~\cite{breloy2018intrinsic}.

\begin{theorem}[Fisher-Rao distance of C-CES distributions]
    The square of the Fisher distance of C-CES distributions over $\manifold{\nfeatures}$ is defined, for all $\point$ and $\pointBis\in\manifold{\nfeatures}$, by
    \begin{equation*}
        \squaredist{\point}{\pointBis}
        = \alpha \|\log(\point^{-1}\pointBis)\|_F^2
        + \beta (\log\det(\point^{-1}\pointBis))^2.
    \end{equation*}
\label{thm:distance_ces}
\end{theorem}
\begin{proof}
    From the proof of Theorem~\ref{thm:geodesics}, $\dot{\gamma}(t)\gamma(t)^{-1}=\tangentVector\point^{-1}$ for all $t\in[0,1]$.
    Thus, we can deduce that $\metric{\gamma(t)}{\dot{\gamma}(t)}{\dot{\gamma}(t)}=\metric{\gamma(0)}{\dot{\gamma}(0)}{\dot{\gamma}(0)}$ for all $t\in[0,1]$.
    Therefore, $\squaredist{\point}{\pointBis}=\metric{\gamma(0)}{\dot{\gamma}(0)}{\dot{\gamma}(0)}$, with $\gamma(0)=\point$ and $\dot{\gamma}(0) = \point\log(\point^{-1}\pointBis)$.
    It follows that
    \begin{equation*}
        \squaredist{\point}{\pointBis} = \alpha \tr((\log(\point^{-1}\pointBis))^2) + \beta (\tr(\log(\point^{-1}\pointBis)))^2.
    \end{equation*}
    To conclude, it is enough to recall that $\tr(\log(\MAT{A}))=\log\det(\MAT{A})$.
\end{proof}

\Remark{
We previously noticed that the Levi-Civita connection and geodesics do not depend on the coefficients $\alpha$ and $\beta$ of the metric.
However, since the Riemannian distance integrates the metric along the geodesics, it does well depend on these factors.
This means that the Fisher-Rao distance (Riemannian distance according to the Fisher in formation geometry) actually depends on the underlying C-CES distribution.
}

\section{Covariance matrix estimation with Riemannian optimization}
\label{sec:Optim}

The estimation of the covariance matrix of a set of observations is a ubiquitous problem in signal processing and machine learning.
Among many applications involving this quantity, we can mention: adaptive filtering and detection, metric learning in classification, data analysis (e.g., graph learning), and dimension reduction.
This section discusses covariance matrix estimation within the class of C-CES, and illustrates how the concepts related to Fisher-Rao information geometry can be leveraged in this context.
First, Section \ref{sec:MLE_reminders} provides some reminders on covariance matrix estimation in the C-CES framework (cf. Section 6 of the background chapter for more details).
Second, section \ref{sec:MLE_RiemOpt} presents an introduction to Riemannian optimization, where maximum likelihood estimation of C-CES models is used as a driving example.
Finally, Section \ref{sec:beyondMLE}
shortly presents how this framework can be leveraged to more general regularized covariance matrix estimation problems and points to references on the matter.

\subsection{Reminders on covariance matrix estimation within CES}
\label{sec:MLE_reminders}

Given a $\nsamples$-sample $ \{\data_i\}_{i=1}^{\nsamples} $ assumed to be iid from $\data\sim\distribution{\point}{\densityGenerator}$, with unknown covariance matrix $\point$, we consider inferring this matrix.
The most common approach to tackle this problem consists in maximizing the log-likelihood function in~\eqref{eq:log_likelihood}.
The maximum likelihood estimator is thus obtained as a solution to the optimization problem
\begin{equation}
    \minimize_{\point\in\manifold{\nfeatures}} \quad \costfunction(\point)
\label{eq:opt_mle}
\end{equation}
where $\costfunction$ denotes in short the negative log-likelihood of the sample set $ \{\data_i\}_{i=1}^{\nsamples} $, i.e.:
\begin{equation}
    \costfunction(\point)
    =
    -\loglikelihood \left(\{\data_i\}_{i=1}^{\nsamples} | \point \right),
    \label{eq:L_short}
\end{equation}
with $\loglikelihood$ defined in~\eqref{eq:log_likelihood}.
The solution of~\eqref{eq:opt_mle} yields the MLE in the form of a fixed point equation
\begin{equation}
\label{eq:MLE}
\pointBis = \frac{1}{\nsamples}  \sum_{i=1}^{\nsamples} 
\psi\left(  \data_i^H \pointBis^{-1} \data_i\right) 
 \data_i  \data_i^H
{=}_d   \mathcal{T}_{\psi} \left( \pointBis \right),
\end{equation}
where $\psi(t) = - g'(t)/g(t)$.
This solution is most commonly evaluated thanks to a fixed-point algorithm
\begin{equation}
    \point_{(k+1)} = 
    \mathcal{T}_{\psi} \left( \point_{(k)} \right)
    = \frac1{\nsamples}\sum_{i=1}^{\nsamples} \psi(\data_i^H\point_{(k)}^{-1}\data_i) \, \data_i \data_i^H.
\label{eq:mle_fixedpoint}
\end{equation}
The existence and uniqueness of the fixed-point solution \eqref{eq:MLE}, as well as the convergence of the fixed-point algorithm \eqref{eq:mle_fixedpoint} is subject subject to conditions on 
the function $\psi$ (resp. the density generator $\densityGenerator$) and the sample set $\{\data_i\}_{i=1}^{\nsamples}$, e.g., obtained in \cite[Theorems 6 and 7]{ollila2012complex}. A notable condition in the absolutely continuous case is that the sample size is required to be larger than the dimension, i.e., $n>p$.

\Remark{
In practice, the true density generator $\densityGenerator$ may not be known or accurately specified.
In the robust estimation theory, an $M$-estimator of the scatter matrix \cite{maronna1976robust, tyler1987distribution} refers to an estimator built from \eqref{eq:MLE}-\eqref{eq:mle_fixedpoint} using a function $\psi(t)$ that is not necessarily linked to the density generator $\densityGenerator$ (cf. Section 6.3 of the background chapter). 
In this chapter, we focus on the example of the MLE, but the tools that will be presented apply to any generic cost function~$\costfunction$.
}

\subsection{Computing MLEs with Riemannian optimization}

\label{sec:MLE_RiemOpt}



%
%
%

Riemannian optimization~\cite{absil2009optimization,boumal2023introduction} is a general framework to solve optimization problems on manifolds.
This extends to Riemannian manifolds classical Euclidean optimization methods such as steepest gradient descent, conjugate gradient, Broyden-Fletcher-Goldfarb-Shanno (BFGS) algorithm, Newton method, trust region, \emph{etc}.
This section introduces Riemannian optimization on $\manifold{\nfeatures}$ as a framework to solve~\eqref{eq:opt_mle} that can leverage tools from the Fisher-Rao information geometry.
At the end of this section, we will see that this framework actually yields the fixed point algorithm~\eqref{eq:mle_fixedpoint} as a special case (specifically, a Riemannian steepest gradient descent with a specific choice of metric, retraction, and step-size).

We consider an optimization problem of the from~\eqref{eq:opt_mle} that has no obvious closed-form solution on $\manifold{\nfeatures}$.
In order to evaluate this solution, we resort to iterative methods, i.e., methods that yield a sequence of iterates $\{\point_{(k)}\}$ in $\manifold{\nfeatures}$ from a starting point $\point_{(0)}\in\manifold{\nfeatures}$.
This sequence is constructed so that it eventually converges to a critical point of the objective in ~\eqref{eq:opt_mle}.
When the variable is constrained to lie in the manifold $\manifold{\nfeatures}$, a generic first-order Riemannian optimization method operates as follows: 
\begin{enumerate}
    \item At iterate $\point_{(k)}\in\manifold{\nfeatures}$, a descent direction in the tangent space, denoted $\tangentVector_{(k)}\in \tangentSpace{\nfeatures}{\point_{(k)}}\simeq\ambientSpace{\nfeatures}$, is computed by leveraging the Riemannian gradient.
    \item The direction descent $\tangentVector_{(k)}$ is used to obtain the next iterate $\point_{(k+1)}$ on $\manifold{\nfeatures}$. This is achieved through a retraction on $\manifold{\nfeatures}$, which is an operator that maps tangent vectors back onto the manifold.
\end{enumerate}
An illustration of such an optimization process is presented in Figure~\ref{fig:optim}, while the design of these two steps is discussed to solve~\eqref{eq:opt_mle} on $\manifold{\nfeatures}$ in the following.
\begin{figure}
    \centering
    \begin{tikzpicture}[scale=2]

\begin{scope}[shift={(1.3,2)},scale=1.5]
	\draw  (1,0) arc (30:150:1.155)
	      plot [smooth, domain=pi:2*pi] ({cos(\x r)},{0.2*sin(\x r)});
	\draw [dotted] plot [smooth, domain=0:pi] ({cos(\x r)},{0.2*sin(\x r)});
	\draw[font=\small] (1.1,0.3) node {$\manifold{\nfeatures}$};

	\draw [line width=1pt, draw=myblue, dashed]  plot [smooth,domain=-0.455:0.405] (1.1*\x, {-0.1*(sin(2*pi*\x r) - sin(2*pi*0.5)) +0.17});

	\coordinate (z) at (-0.2,0.27);
	\draw [fill=gray!20,opacity=0.4] (z)++(0.5,0.15) -- ++(-0.4,-0.5) -- ++(-0.5,+0.25) -- ++(+0.4,+0.5) -- cycle;
	\node[font=\footnotesize] at (-0.5,0.65) {$\tangentSpace{\nfeatures}{\point_{(k)}}$};

    \draw[draw=myred, dashed, line width=0.8pt] (z)++(0.25,-0.05) -- (0.1,0.12);
 
	\draw[line width=1pt,->,>=stealth,draw=myorange,fill=myorange] (z) -- ++(0.25,-0.05) node[very near end,above,font=\footnotesize,text=myorange] at (z) {$\tangentVector_{(k)}$}; 
	\draw[font=\scriptsize] (z) node {$\bullet$};
	\node[above,font=\footnotesize] at (z) {$\point_{(k)}$};

    \node[font=\scriptsize] at (0.1,0.12) {$\bullet$};
    \node[font=\footnotesize, anchor=south west] at (0.1,0.08) {$\point_{(k+1)}$};

    \coordinate (x) at (-0.5,0.2);
    \node[font=\scriptsize] at (x) {$\bullet$};
    \node[font=\footnotesize,anchor=north west] at (x) {$\point_{(0)}$};

    \coordinate (y) at (0.45,0.12);
    \node[font=\scriptsize] at (y) {$\bullet$};
    \node[font=\footnotesize,anchor=north west] at (y) {$\point_{*}$};
	
	\draw[dotted,thin] (x) -- ++ (0,-0.93);
	\draw[dotted,thin] (y) -- ++ (0,-1.26);

	\draw[dotted,thin] (z) -- ++ (0,-1.22);

\end{scope}

\begin{axis}[
	width=0.365\linewidth,
	height=80,
	draw=black,
    line width=0.2pt,
	axis lines=center,
	axis on top=true,
	xmin=0,
	xmax=10,
	xtick={\empty},
	ymin=0,
	ymax=8.5,
	ylabel={},
	ylabel style ={at={(axis description cs:-0.02,.9)},anchor=east},
	yticklabel={\empty},
	ytick={\empty},
]
	\addplot[mark=*,mark size=0.5pt, only marks,draw=black,fill=black] coordinates {
        (2.05,6.2)
		(3.7,4.0)
        (7.35,2.0)
	};
	\addplot[domain=2.05:7.35, samples=10, draw=myblue, dashed, line width=0.5pt]{3/20*(x-7.35)^2+2};
	\addplot[color=myorange,domain=2.0:6.0, samples=10, line width=0.5pt]{-1.1*(x-3.7)+3.92};
\end{axis}

\node[font=\scriptsize,text=myorange, rotate=-30] at (1,0.4) {$\Diff \costfunction(\point_{(k)})[\tangentVector_{(k)}]$};

\node[font=\scriptsize,anchor=east] at (0.57,1.05) {$\costfunction(\point_{(0)})$};

\node[font=\scriptsize,anchor=west] at (0.96,0.66) {$\costfunction(\point_{(k)})$};

\node[font=\scriptsize,anchor=west] at (1.95,0.4) {$\costfunction(\point_{*})$};

\end{tikzpicture}
    \caption{Illustration of Riemannian optimization.
    Given some initialization $\point_{(0)}$, the goal is to reach the minimum $\point_*$.
    At $\point_{(k)}$, the descent direction $\tangentVector_{(k)}$ is such that it induces a decrease in $\costfunction$, \emph{i.e.}, $\Diff \costfunction(\point_{(k)})[\tangentVector_{(k)}]<0$ (slope of the orange line).}
    \label{fig:optim}
\end{figure}
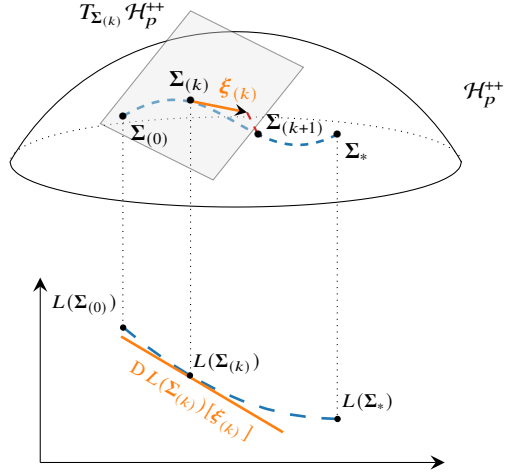

For the first step, the steepest descent direction is given by the gradient, which is defined through the metric in the Riemannian setting.
The Riemannian gradient of the negative log-likelihood $\costfunction$ at $\point\in\manifold{\nfeatures}$ according to the metric of Theorem~\ref{thm:fim_metric} is the unique tangent vector $\grad \costfunction(\point)\in\tangentSpace{\nfeatures}{\point}\simeq\ambientSpace{\nfeatures}$ such that, for all $\tangentVector\in\ambientSpace{\nfeatures}$, we~have
\begin{equation}
    \metric{\point}{\grad \costfunction(\point)}{\tangentVector}
    = \Diff \costfunction(\point)[\tangentVector].
\end{equation}
This Riemannian gradient is provided in Proposition~\ref{prop:rgrad}.
%

\begin{proposition}[Riemannian gradient of $\costfunction$]
    The Riemannian gradient $\grad \costfunction(\point)$ of the negative log-likelihood $\costfunction$ defined in~\eqref{eq:L_short} at $\point\in\manifold{\nfeatures}$ according to metric~\eqref{eq:AI_metric}~is
    \begin{multline*}
        \grad\costfunction(\point)
        = \left( \frac{\nsamples}{\alpha+p\beta} + \frac{\beta}{\alpha(\alpha+p\beta)}\sum_{i=1}^\nsamples \psi(\data_i^H\point^{-1}\data_i) \, \data_i^H\point^{-1}\data_i \right) \point
        \\
        - \frac1\alpha \sum_{i=1}^{\nsamples} \psi(\data_i^H\point^{-1}\data_i) \, \data_i\data_i^H. 
    \end{multline*}
\label{prop:rgrad}
\end{proposition}
\begin{proof}
    From the beginning of the proof of Theorem~\ref{thm:fim_metric}, we get that the directional derivative of $\costfunction$ at $\point\in\manifold{\nfeatures}$ in direction $\tangentVector\in\ambientSpace{\nfeatures}$ is
    \begin{equation*}
        \begin{array}{rcl}
            \Diff\costfunction(\point)[\tangentVector] & = & \nsamples\tr(\point^{-1}\tangentVector) + \sum_{i=1}^{\nsamples} \frac{g'}g(\data_i^H\point^{-1}\data_i) \tr(\point^{-1}\tangentVector\point^{-1}\data_i\data_i^H)
             \\
             & = & \tr\left(\point^{-1} \left( n\point - \sum_{i=1}^{\nsamples} \psi(\data_i^H\point^{-1}\data_i\right) \data_i\data_i^H ) \point^{-1}\tangentVector \right).
        \end{array}
    \end{equation*}
    Thus, for $\alpha=1$ and $\beta=0$, we immediately get the result by identification.
    To obtain the result in the general case, notice that, given $\MAT{A}$ and $\MAT{B}$ in $\ambientSpace{\nfeatures}$, if we set $\MAT{\tilde{A}} = \frac1\alpha\MAT{A} - \frac{\beta}{\alpha(\alpha+p\beta)}\tr(\MAT{A})\MAT{I}_{\nfeatures}$, then we have $\tr(\MAT{A}\MAT{B})=\alpha\tr(\MAT{\tilde{A}}\MAT{B}) + \beta\tr(\MAT{\tilde{A}})\tr(\MAT{B})$.
    Taking $\MAT{A} = \point^{-\nicefrac12} ( n\point - \sum_{i=1}^n \psi(\data_i^H\point^{-1}\data_i ) \data_i\data_i^H ) \point^{-\nicefrac12}$ and $\MAT{B}=\point^{-\nicefrac12}\tangentVector\point^{-\nicefrac12}$, and basic calculations allow to conclude.
\end{proof}

To perform the second step, it remains to define a retraction that maps tangent vectors back onto the manifold.
Formally, given $\point\in\manifold{\nfeatures}$, a retraction is a mapping $R_{\point}:\tangentSpace{\nfeatures}{\point}\simeq\ambientSpace{\nfeatures}\to\manifold{\nfeatures}$ such that, for all $\tangentVector\in\ambientSpace{\nfeatures}$,
\begin{equation}
    R_{\point}(\tangentVector) = \point + \tangentVector + o(\|\tangentVector\|).
\label{eq:retr_def}
\end{equation}
From a geometric point of view, the Riemannian exponential mapping
provides the ideal retraction for a manifold equipped with a Riemannian metric (in the sense that it is the most reflective of the considered geometry).
In our case, it is defined in~\eqref{eq:Rexp} and illustrated in Figure~\ref{fig:geodesic_exp_log}.
However, this retraction involves computing the matrix exponential of some Hermitian matrix, which can be computationally costly and/or numerically unstable, as the exponential tends quickly to infinity or zero.
From a practical point of view, it might thus be more advantageous to employ alternate retractions. 
Notice that~\eqref{eq:retr_def} means that a proper retraction is (at least) a first-order approximation of the Riemannian exponential mapping.
Since $\manifold{\nfeatures}$ is open in $\ambientSpace{\nfeatures}$, a proper first order approximation is simply obtained as
\begin{equation}
    R^{(1)}_{\point}(\tangentVector) = \point + \tangentVector.
\label{eq:retr_1st_order}
\end{equation}
The main limitation of $R^{(1)}$ is that, given $\point\in\manifold{\nfeatures}$, there are many $\tangentVector\in\ambientSpace{\nfeatures}$ such that $R^{(1)}_{\point}(\tangentVector)\notin\manifold{\nfeatures}$.
This means that the iterative algorithms that employ this retraction are not guaranteed to be numerically stable.
%
%
To overcome this issue, Proposition~\ref{prop:retr_2nd_order} provides a retraction that is a second-order approximation of the Riemannian exponential~\eqref{eq:Rexp} (initially proposed in~\cite{jeuris2012survey}), that does not suffer the same limitation as $R^{(1)}$.
%

\begin{proposition}[Second order retraction]
    The retraction $R^{(2)}$ such that, for all $\point\in\manifold{\nfeatures}$ and $\tangentVector\in\ambientSpace{\nfeatures}$,
    \begin{equation*}
        R^{(2)}_{\point}(\tangentVector) = \point + \tangentVector + \frac12\tangentVector\point^{-1}\tangentVector
    \end{equation*}
    is a second order approximation of the Riemannian exponential mapping~\eqref{eq:Rexp}.
    Furthermore, for all $\point\in\manifold{\nfeatures}$ and $\tangentVector\in\ambientSpace{\nfeatures}$, $R^{(2)}_{\point}(\tangentVector)$ belongs to $\manifold{\nfeatures}$.
\label{prop:retr_2nd_order}
\end{proposition}

\begin{proof}
    Recall that the matrix exponential of $\MAT{A}$ is $\exp(\MAT{A})=\sum_{k=0}^{\infty} \frac{\MAT{A}^k}{k!}$.
    Hence the second order approximation is $\exp(\MAT{A})=\MAT{I}_{\nfeatures}+\MAT{A}+\frac12\MAT{A}^2+o(\|\MAT{A}\|^2)$.
    Applying this to~\eqref{eq:Rexp}, we obtain $\exp_{\point}(\tangentVector)=\point(\MAT{I}_{\nfeatures}+\point^{-1}\tangentVector+\frac12\point^{-1}\tangentVector\point^{-1}\tangentVector) + o(\|\tangentVector\|^2)$.
    Basic calculations yield the result.
    Moreover, it is obviously a proper retraction.
    It remains to show that we always get a matrix in $\manifold{\nfeatures}$.
    To do so, notice that 
    \begin{equation*}
        R^{(2)}_{\point}(\tangentVector) = \point^{\nicefrac12}(\MAT{I}_{\nfeatures} + \point^{-\nicefrac12}\tangentVector\point^{-\nicefrac12} + \frac12(\point^{-\nicefrac12}\tangentVector\point^{-\nicefrac12})^2)\point^{\nicefrac12}.
    \end{equation*}
    Let the eigenvalue decomposition $\point^{-\nicefrac12}\tangentVector\point^{-\nicefrac12}=\MAT{U}\MAT{\Lambda}\MAT{U}^H$.
    Then 
    \begin{equation*}
        R^{(2)}_{\point}(\tangentVector)=\point^{\nicefrac12}\MAT{U}(\MAT{I}_{\nfeatures} + \MAT{\Lambda} + \frac12\MAT{\Lambda}^2)\MAT{U}^H\point^{\nicefrac12}.
    \end{equation*}
    The result follows from the fact that the second order polynomial $\lambda\mapsto 1+\lambda+\frac12\lambda^2$ is strictly positive for all values of $\lambda$.
\end{proof}

We now have everything needed to define an iterative algorithm that solves the MLE optimization problem~\eqref{eq:opt_mle}.
Given the retraction $R$, we can, for instance, define the Riemannian gradient descent that yields the sequence of iterates
\begin{equation}
    \point_{(k+1)} = R_{\point_{(k)}}(-\lambda_k\grad \costfunction(\point_{(k)})),
\label{eq:RGD_mle}
\end{equation}
where $\lambda_k$ is the step size, which can be set by the user or computed through a line search; see e.g.~\cite{absil2009optimization,boumal2023introduction}.

Our final point in this section is to show that the fixed point algorithm~\eqref{eq:mle_fixedpoint} is, in fact, a particular case of~\eqref{eq:RGD_mle}.
Indeed, if we choose $\alpha=1$ and $\beta=0$, the Riemannian gradient of Proposition~\ref{prop:rgrad} is
\begin{equation*}
    \grad \costfunction(\point) = n\point - \sum_{i=1}^n \psi(\data_i^H\point^{-1}\data_i) \, \data_i\data_i^H.
\end{equation*}
Algorithm~\eqref{eq:mle_fixedpoint} is then obtained from~\eqref{eq:RGD_mle} by choosing the first order retraction~\eqref{eq:retr_1st_order} and constant step size $\lambda_k=\frac1n$.
Notice that in this particular case, the choice of the first order retraction~\eqref{eq:retr_1st_order} is a valid choice because the particular structure of the gradient ensures that all iterates remain in $\manifold{\nfeatures}$.
Though alternate choices of $\alpha$ and $\beta$ in the metric, step size, and retraction could improve the convergence speed in some cases, this fixed-point is generally a good all-purpose candidate to compute MLEs as in \eqref{eq:MLE}.
However, having recast it from the prism of Riemannian geometry opens many perspectives, which are discussed in the next section.

\subsection{Beyond MLE and fixed-point algorithms}
\label{sec:beyondMLE}

The MLEs (and $M$-estimators) are known for their good asymptotic performance in terms of estimation accuracy~\cite{couillet2015random, dravskovic2019asymptotics, dravskovic2018new,ollila2011complex, zhang2014marchenko}.
Still, they suffer from two limitations:
$i$) they do not exist when the sample set is lower than the dimension ($\nsamples<\nfeatures$); 
$ii$) they can be inaccurate when $\nsamples\simeq\nfeatures$, as they do not leverage any bias-variance trade-off improvement.
These limitations motivated the development of generalized estimation procedures by expressing new estimators as solutions to penalized optimization problems of the form:
\begin{equation} \label{eq:reg_mest}
     \minimize_{\point\in\manifold{\nfeatures}} \quad   
     \costfunction(\point)
    + \lambda h(\point)
\end{equation}
where $\costfunction$ is the negative log-likelihood as in~\eqref{eq:L_short}, $\lambda\in\realSpace^+$ is a regularization parameter, and $h$ is a penalty function that promotes some form of regularization.
Among many options considered in the literature for $h$, we can mention shrinkage to a target matrix~\cite{ollila2014regularized,pascal2014generalized,sun2014regularized}, shrinkage of the eigenvalues~\cite{Wiesel2012unified,breloy2019spectral}, promoting a sparse graphical structure~\cite{hippert2022learning,zhang2013multivariate}, or pooling from groups of observations~\cite{collas2022robust,ollila2016simultaneous}.
For appropriate choices of regularization penalty and parameters, the regularized estimators, as formulated in~\eqref{eq:reg_mest}, can overcome the aforementioned issues of their non-regularized counterparts.
In this scope, the Riemannian geometry provides useful tools to address and study~\eqref{eq:reg_mest}, which is discussed next.

\subsubsection{Riemannian options for computing solutions of \eqref{eq:reg_mest}}  

The optimization problems expressed in \eqref{eq:reg_mest} generally do not exhibit closed-form or fixed-point solutions and, thus, require the use of iterative algorithms to be evaluated.
In this setup, the Riemannian optimization framework is a good candidate in order to ensure that the variable remains in $\manifold{\nfeatures}$ along the iterations.
Beyond the introduction of the Riemannian gradient descent presented in Section \ref{sec:MLE_RiemOpt}, this flexible framework extends to many other algorithms:
\begin{itemize}
    \item Conjugate gradient, or BFGS-type algorithms, require the notion of Riemannian vector transport operator~\cite[Section 10.3]{boumal2023introduction}, which allows to transport tangent vectors between tangent spaces at different points.

    \item 
    Second-order methods, such as trust region or Newton methods
    require the definition of the Riemannian Hessian~\cite[Section 5.5]{boumal2023introduction}.

    \item For large dimensional datasets, stochastic optimization methods can also be extended to the Riemannian setting \cite{bonnabel2013stochastic, zhang2016riemannian, bouchard2020riemannian}.
    
\end{itemize}
A last remark is that in these algorithms, the metric is left as a choice that conditions the gradient and possibly the retraction.
There are various options for $\manifold{\nfeatures}$ (cf. Section \ref{sec:hpd_as_riem}), with their respective pros and cons.
It is still noticed that the gradient obtained from the Fisher information metric, also referred to as the natural gradient \cite{amari1998natural}, is generally experienced to yield a faster convergence when dealing with a cost function related to the statistical model of the data (cf. examples in \cite{han2021riemannian, collas2023riemannian}).

\subsubsection{Geodesic convexity on $\manifold{\nfeatures}$}

The classical results on the existence and uniqueness of the MLEs \cite[Theorems 6 and 7]{ollila2012complex} do not directly extend to the formulation in \eqref{eq:reg_mest}, so one might inquire about the optimally of the solution obtained by reaching a local minimum of this problem.
In this scope, the Riemannian perspective offers some answers by generalizing the property of convexity.
First, we recall that the geometry induced by the Fisher information metric \eqref{eq:AI_metric} yields geodesic curves $\gamma(t)$ as defined in Theorem \ref{thm:geodesics} between any two points $\point_0,\point_1 \in \manifold{\nfeatures}$.
A function $f$ is then said to be geodesically convex ($g$-convex) on $\manifold{\nfeatures}$ if $\forall \point_0, \point_1 \in \manifold{\nfeatures}$, it satisfies the inequality
\begin{equation} \label{eq:gcvx}
f (\gamma(t))
\leq
(1-t)
f (\point_0)
+
t
f (\point_1),~\forall t\in [0,1].
\end{equation}
If the above inequality is strict, the function is then said to be strictly $g$-convex.
The $g$-convexity enjoys properties similar to those of the convexity in the standard Euclidean case, in particular:
\begin{theorem}(Global minimizer of $g$-convex functions on $\manifold{\nfeatures}$).\\
Let $f: \manifold{\nfeatures} \to \realSpace$ be $g$-convex as defined in \eqref{eq:gcvx}, then any local minimum of $f$ over $ \manifold{\nfeatures} $ is a global minimum. Furthermore, if $f$ is strictly $g$-convex, this global minimum is unique.
\end{theorem}
This property offers an alternate proof for the uniqueness of MLEs as in \eqref{eq:opt_mle} \cite{ollila2014regularized}, and had practical impacts for the design of regularized covariance matrix estimators as in \eqref{eq:reg_mest}: many examples of penalty functions (with various regularization effects) can be found in the overviews in \cite{wiesel2015structured, duembgen2016geodesic}, and the references \cite{auderset2005angular, Wiesel2012unified, Wiesel2012geodesic, wiesel2015structured, ollila2014regularized, duembgen2016geodesic}.

\section{Intrinsic Cramér-Rao Bound for covariance matrix estimation}

\label{sec:iCRLB}

The Cramér-Rao inequality is a staple tool in statistics that characterizes the optimal mean-squared error an unbiased estimator can reach given a model and setup \cite{kay1993fundamentals}.
This tool can either be used to validate estimation procedures, or to design systems so that a certain level of accuracy is guaranteed to be theoretically reachable.
While the Euclidean formulation of this inequality was briefly introduced in Section \ref{sec:background_ces}, the so-called \textit{intrinsic} Cram\'er-Rao bounds extend it to parameters living in a manifold, and for any chosen Riemannian metric.
This perspective is especially interesting as: 
$i$) some metrics can be more meaningful to assess the estimation performance in a given application; $ii)$ a suitable Riemannian geometry (as opposed to the Euclidean one) can reveal hidden properties that make the bound more informative (such as curvature terms, intrinsic biases, etc.).
First, Section \ref{sec:icrb_background} introduces the background on intrinsic Cramér-Rao bound from \cite{smith2005covariance}, where the C-CES model is used as a driving example.
We also refer the reader to \cite[Chapter 6]{boumal2014thesis} and the reference \cite{barrau2013note}, for more details on the topic.
Then, Cramér-Rao bounds are derived for various distances in the context of covariance matrix estimation within C-CES distributions \cite{breloy2018intrinsic} in Section \ref{sec:icrb_CES}.

\subsection{Introduction to intrinsic Cramér-Rao bounds}

\label{sec:icrb_background}

This subsection will present tools that can be applied to any chosen Riemannian geometry on $\manifold{\nfeatures}$.
The needed objects are the Riemannian metric, logarithm mapping and square of the distance, which are denoted $\metricGeneric{\cdot}{\cdot}{\cdot}$, $\logGeneric_{\cdot}(\cdot)$ and $\squaredistGeneric{\cdot}{\cdot}$, respectively.
As in Section \ref{sec:Optim}, we consider the problem of estimating the matrix $\point$ from a given $\nsamples$-sample $\{\data_i\}_{i =1}^{\nsamples}$ assumed to be iid from $\data\sim\distribution{\point}{\densityGenerator}$.
We denote $\estimator$ an estimator of this parameter; e.g. the MLE presented in Section~\ref{sec:Optim}.
We then consider the evaluation of the performance of such estimator $\estimator$.
To do so, we exploit the chosen Riemannian metric $\metricGeneric{\point}{\cdot}{\cdot}$.
Such metric can, for example, be the Fisher information one~\eqref{eq:AI_metric}, or one of the many other options from the literature \cite{thanwerdas2023n}.
The performance criterion is the resulting square of the Riemannian distance, i.e., the error is measured through $\squaredistGeneric{\point}{\estimator}$.
The intrinsic Cramér-Rao theory from~\cite{smith2005covariance} then allows us to obtain a lower bound on the expectation of this error for any unbiased estimator $\estimator$.
Eventually, this retrieves the well-known inequality `` $\mathbf{C}\succeq \mathbf{F}^{-1}$,'' with $\mathbf{C}\in\realSpace^{\nfeatures^2\times\nfeatures^2}$ being the covariance matrix of the estimation error and $\mathbf{F}\in\realSpace^{\nfeatures^2\times \nfeatures^2}$ being the Fisher information matrix, where $\nfeatures^2=\dim(\manifold{\nfeatures})$.
However, these parameters have different definitions due to the specific nature of the considered objects.
The point of this section is to briefly present the key ingredients to obtain such inequality and the corresponding main theorem.

First, we need to generalize the notion of estimation error vector $\errVec\in\realSpace^{\nfeatures^2}$ to the Riemannian context.
Notice that, in the Euclidean case, such vector is generally constructed by vectorizing the entry-wise subtraction of the covariance matrix $\point$ to its estimate $\estimator$, i.e., $\errVec^{\mathcal{E}}=\vech(\estimator-\point)$, where $\vech(\cdot)$ denotes the half-vectorization operator.
%
%
As it happens, from a Riemannian geometry point of view, $\estimator-\point$ corresponds to the Euclidean logarithm mapping at $\point$.
Therefore, the Riemannian logarithm $\logGeneric_{\point}(\estimator)$ provides a natural way to extend the error to any geometry.
It is indeed an element of the tangent space $\tangentSpace{\nfeatures}{\point}$ of $\point$ that ``points towards'' $\estimator$, and whose norm corresponds to the Riemannian distance.
It remains to actually get an error vector $\errVec\in\realSpace^{\nfeatures^2}$ from $\logGeneric_{\point}(\estimator)$.
To do so, we leverage a basis $\{\tangentVector_q  \}_{q=1}^{\nfeatures^2}$ of $\tangentSpace{\nfeatures}{\point}\simeq\ambientSpace{\nfeatures}$ that is orthonormal with respect to the chosen metric $\metricGeneric{\point}{\cdot}{\cdot}$.
%
%
In practice, such a basis can be obtained either analytically from mathematical calculations or numerically, thanks to the Gram-Schmidt orthonormalization process.
This basis yields the decomposition
\begin{equation}
    \logGeneric_{\point}(\estimator)  = \sum_{q=1}^{\nfeatures^2} \varepsilon_q \tangentVector_q,
\end{equation}
and we denote $\errVec=[\varepsilon_1,\cdots,\varepsilon_{\nfeatures^2}] \in \mathbb{R}^{\nfeatures^2}$ the corresponding coordinates error vector, obtained as
\begin{equation}
    \varepsilon_q = \;
    \metricGeneric{\point}{\logGeneric_{\point}(\estimator)}{\tangentVector_q}.
    \label{eq:err_vec2}
\end{equation}
Moreover, the norm of this vector corresponds to the Riemannian distance between $\point$ and $\estimator$, i.e.,
\begin{equation}
    \squaredistGeneric{\point}{\estimator}
    = \,
    \metricGeneric{\point}{\logGeneric_{\point}(\estimator)}{\logGeneric_{\point}(\estimator)}
    =
    \| \errVec \|_2^2,
\end{equation}
which will be instrumental in the next derivations.
The basis $\{\tangentVector_q\}_{q=1}^{\nfeatures^2}$ also yields a Fisher information matrix $\mathbf{F}$, with entries
\begin{equation}
    \mathbf{F}_{q\ell} = \,
    \metric{\point}{\tangentVector_q}{\tangentVector_\ell}^{\textup{FIM}}.
    \label{eq:matrix_fim_for_icrlb}
\end{equation}
The matrix $\mathbf{F}$ represents the Fisher information metric of Theorem~\ref{thm:fim_metric} according to this system of coordinates.
Then, from~\cite[Corrolary 2]{smith2005covariance}, we obtain Theorem~\ref{thm:icrlb}.
\begin{theorem}[Intrinsic Cramér-Rao bound]
    Let $\point\in\manifold{\nfeatures}$.
    Let $\{\data_i\}_{i=1}^{\nsamples}$ a iid $\nsamples$-sample from $\data\sim\distribution{\point}{\densityGenerator}$. 
    Let $\estimator$ an unbiased estimator of $\point$ with corresponding error vector $\errVec$ defined in~\eqref{eq:err_vec2}.
    %
    Then
    \begin{equation*}
        \mathbf{C} = \expectation{\errVec\errVec^T}
        \succeq
        \mathbf{F}^{-1} +
        \textup{ curvature terms},
    \end{equation*}
    where $\mathbf{F}$ is the Fisher information matrix in \eqref{eq:matrix_fim_for_icrlb} and the curvature terms -- which are not detailed here -- depend on the Riemannian curvature tensor corresponding to the chosen geometry and on $\mathbf{F}$; see~\cite{smith2005covariance,boumal2013intrinsic,boumal2014thesis} for further details.
\label{thm:icrlb}
\end{theorem}
In practice, the curvature terms can usually be neglected in Theorem~\ref{thm:icrlb}.
Furthermore, taking the trace of the inequality yields the desired result, i.e.,
\begin{equation}
    \expectation{\squaredistGeneric{\point}{\estimator}}
    \geq
    \tr(\mathbf{F}^{-1}).
\label{eq:crb_on_dist}
\end{equation}
It offers a bound that can be derived for any chosen Riemannian distance $\squaredistGeneric{\cdot}{\cdot}$ (and corresponding metric $\metricGeneric{\cdot}{\cdot}{\cdot}$).

\Remark{
The inequality in Theorem \ref{thm:icrlb} interestingly takes into account the curvature of the manifold, which, for $\manifold{\nfeatures}$, only depends on the chosen metric.
In the Euclidean case, such curvature term is null, and we recover the standard Cramér-Rao inequality.
We also notice that the theorem in~\cite{smith2005covariance} also incorporates an intrinsic bias terms, which was excluded here for the sake of conciseness.
This intrinsic bias (expectation of the Riemannian logarithm) depends on the estimator and the chosen metric, and can reveal unexpected properties.
A main example is that the MLE of the covariance matrix of the Gaussian model appears unbiased in the Euclidean setting, but is, in fact, biased when using the Fisher information metric~\cite{smith2005covariance}.
Such analysis thus opens prospects for improved estimation from the intrinsic perspective.
}

\subsection{Bounds for various matrix distances in C-CES distributions}

\label{sec:icrb_CES}

This Section presents the derivation of special cases of Theorem~\ref{thm:icrlb} when considering various usual metrics.
Hence, it yields intrinsic Cramér-Rao bounds for the problem of covariance matrix estimation in C-CES distributions for the corresponding Riemannian distances.
Since the Fisher information metric is already obtained in Theorem~\ref{thm:fim_metric}, the derivation boils down to the following steps:
\begin{itemize}
\item[$a)$] Selecting the performance metric $\metricGeneric{\point}{\cdot}{\cdot}$ and computing $\{\tangentVector_q\}_{q=1}^{\nfeatures^2}$, a corresponding orthonormal basis of $\tangentSpace{\nfeatures}{\point}$;

\item[$b)$] Computing the elements of the  Fisher information matrix with this basis, according to~\eqref{eq:matrix_fim_for_icrlb};

\item[$c)$] Inverting the Fisher information matrix, applying Theorem \ref{thm:icrlb}, then \eqref{eq:crb_on_dist}.
\end{itemize}
These operations are conducted in the following for the Euclidean metric, the so-called natural Riemannian metric (the affine invariant metric~\eqref{eq:AI_metric} with $\alpha=1$ and $\beta=0$), and the Fisher-Rao metric of the assumed model (i.e., the metric of Theorem \ref{thm:fim_metric}: \eqref{eq:AI_metric} with $\alpha=\alpha_{\densityGenerator}$ and $\beta=\beta_{\densityGenerator}$, where $\alpha_{\densityGenerator}$ and $\beta_{\densityGenerator}$ are defined in~\eqref{eq:coefficients_metric}).
In order for the chosen values of $\alpha$ and $\beta$ to be clear, in this subsection, the metric~\eqref{eq:AI_metric} is denoted $\metric{\cdot}{\cdot}{\cdot}^{(\alpha,\beta)}$ and the distance of Theorem~\ref{thm:distance_ces} is denoted $\delta_{(\alpha,\beta)}^2(\cdot,\cdot)$.

\subsubsection{Euclidean distance}
\label{sec:crb_scatter_eucl}
We first recall the elementary tools of the Euclidean metric for $\manifold{\nfeatures}$:
\begin{equation}
\begin{array}{lcl}
    \text{Metric:} & ~ &
    \metric{\point}{\tangentVector}{\tangentVectorBis}^{\mathcal{E}}=\reel(\tr(\tangentVector\tangentVectorBis))
    \\[0.1cm]
    \text{Logarithm:} & ~ &
    \log^{\mathcal{E}}_{\point}(\estimator)=\estimator-\point
    \\[0.1cm]
    \text{Distance:} & ~ &
    \delta_{\mathcal{E}}^2(\point,\estimator)=\|\estimator-\point\|_2^2.
\end{array}
\label{eq:recall_eucl_metric}
\end{equation}
A basis of the tangent space $\tangentSpace{\nfeatures}{\point}$ that is orthonormal with respect to the metric in~\eqref{eq:recall_eucl_metric} can be obtained as follows:
\begin{enumerate}
\item For $1\leq i\leq\nfeatures$, $\tangentVector^{\mathcal{E}}_{ii}$ is a $\nfeatures\times\nfeatures$ symmetric matrix whose $i^{\textup{th}}$ diagonal element is one, zeros elsewhere
\item For $1\leq i<j\leq\nfeatures$, $\tangentVector^{\mathcal{E}}_{ij}$ is a $\nfeatures\times\nfeatures$ symmetric matrix whose $ij^{\textup{th}}$ and $ji^{\textup{th}}$ elements are both $\nicefrac{1}{\sqrt{2}}$, zeros elsewhere.
\item For $1\leq i<j\leq\nfeatures$, $\bar{\tangentVector}^{\mathcal{E}}_{ij}$ is a $\nfeatures\times\nfeatures$ Hermitian matrix whose $ij^{\textup{th}}$  and $ji^{\textup{th}}$ elements are $\nicefrac{\sqrt{-1}}{\sqrt{2}}$ and $\nicefrac{-\sqrt{-1}}{\sqrt{2}}$, respectively, zeros elsewhere.
\end{enumerate}
To shorten notations, we simply denote this basis $\{\tangentVector_q^{\mathcal{E}}\}_{q=1}^{\nfeatures^2},$ where the $\nfeatures^2$ elements are ordered following items 1), 2), and 3). 
The squared Euclidean distance between an estimator $\estimator$ and the true value $\point$ also corresponds to the summed squared errors on the coordinates in this basis.
We then have the following result:
\begin{theorem}[Cram\'er-Rao bound on Euclidean distance]
    Let $\estimator$ an unbiased estimator of $\point$ built from iid data $\{\data_i\}_{i=1}^{\nsamples}$ drawn from $\data\sim\distribution{\point}{\densityGenerator}$.
    The Euclidean distance between $\estimator$ and $\point$ is bounded in expectation as
    \begin{equation*}
        \mathbf{C}_{\mathcal{E}} = 
        \expectation{\delta_{\mathcal{E}}^2(\estimator,\point)}
        \geq
        \tr(\mathbf{F}_{\mathcal{E}}^{-1}),
    \end{equation*}
    where
    \begin{equation*}
        \left[\mathbf{F}_{\mathcal{E}}\right]_{q\ell} =
        \reel(
        n\alpha_{\densityGenerator} \tr(\point^{-1}\tangentVector^{\mathcal{E}}_q\point^{-1}\tangentVector^{\mathcal{E}}_{\ell}))
        +
        n\beta_{\densityGenerator} \tr(\point^{-1}\tangentVector^{\mathcal{E}}_q) \tr(\point^{-1}\tangentVector^{\mathcal{E}}_{\ell})
        ),
    \end{equation*}
    with $\alpha_{\densityGenerator}$ and $\beta_{\densityGenerator}$ defined in~\eqref{eq:coefficients_metric}.
\ref{thm:fim_metric}.
\label{thm:eucl_crb_ces}
\end{theorem}
\begin{proof}
The result is a direct application of Theorem \ref{thm:icrlb} and \eqref{eq:crb_on_dist} using the basis $\{\tangentVector_j^{\mathcal{E}}\}_{j=1}^{p^2}$.
\end{proof}
Remark that this corresponds to the Euclidean Cram\'er-Rao bounds obtained for several distributions in \cite{greco2013cramer, pascal2010statistical, besson2013fisher,mitchell1989information}. Also notice that we retrieve the same result as \cite[Theorem 5]{smith2005covariance} for the Gaussian distribution, i.e., $\alpha_{\densityGenerator}=1$ and $\beta_{\densityGenerator}=0$.

\subsubsection{Natural Riemannian distance}
\label{sec:crb_scatter_nat}

The natural Riemannian distance refers to the distance induced by the affine invariant metric~\eqref{eq:AI_metric} with the standard choice of coefficients $\alpha=1$ and $\beta=0$.
The elementary tools for this metric for $\manifold{\nfeatures}$ are
\begin{equation}
\begin{array}{lcl}
    \text{Metric:} & ~ &
    \metric{\point}{\tangentVector}{\tangentVectorBis}^{(1,0)}=\reel(\tr(\point^{-1}\tangentVector\point^{-1}\tangentVectorBis) )
    \\[0.1cm]
    \text{Logarithm:} & ~ & 
    \log_{\point}(\estimator)
    \; = \; \point\log(\point^{-1}\estimator)
    \\[0.1cm]
    \text{Distance:} & ~ & 
    \delta_{(1,0)}^2(\point,\estimator)  = \|\log(\point^{-1}\estimator)\|_2^2.
\end{array}
\label{eq:recall_nat_metric}
\end{equation}
Recall that the full description of this geometry is provided in Section~\ref{sec:RiemGeo}.
A basis of the tangent space $\tangentSpace{\nfeatures}{\point}$ that is orthonormal with respect to the metric in~\eqref{eq:recall_nat_metric} can be obtained by coloring the canonical basis of previous section as
\begin{equation}
\label{eq:nat_basis}
\tangentVector_q^{(1,0)} = \point^{\nicefrac12}\tangentVector_q^{\mathcal{E}}\point^{\nicefrac12}.
\end{equation}
The whole basis is denoted $\{\boldsymbol{\xi}_q^{(1,0)}\}_{q=1}^{\nfeatures^2}$. 
%
%
We then have the following result:
\begin{theorem}[Cram\'er-Rao bound on natural Riemannian distance]
    Let $\estimator$ an unbiased estimator of $\point\in\manifold{\nfeatures}$ built from iid data $\{\data_i\}_{i=1}^{\nsamples}$ drawn from $\data\sim\distribution{\point}{\densityGenerator}$.
    The Riemannian distance between $\estimator$ and $\point$ is bounded in expectation as
    \begin{equation}
        \expectation{\delta_{(1,0)}^2(\estimator,\point)}
        \geq
        \frac1{\nsamples} \left( \frac{\nfeatures^2-1}{\alpha_{\densityGenerator}} +  \frac{1}{\alpha_{\densityGenerator}+ \nfeatures \beta_{\densityGenerator}} \right),
    \label{eq:intrinsic_crb_ces}
    \end{equation}
    with $\alpha_{\densityGenerator}$ and $\beta_{\densityGenerator}$ defined in~\eqref{eq:coefficients_metric}.
\label{thm:icrb_nat}
\end{theorem}
\begin{proof}
    Plugging the basis $\{\tangentVector_q^{(1,0)}\}_{q=1}^{\nfeatures^2}$ of $\tangentSpace{\nfeatures}{\point}$ defined in~\eqref{eq:nat_basis} in~\eqref{eq:matrix_fim_for_icrlb} yields
    \begin{equation*}
        \left[ \mathbf{F}_{(1,0)} \right]_{q\ell}
        = \metric{\point}{\tangentVector^{(1,0)}_q}{\tangentVector^{(1,0)}_{\ell}}^{\textup{FIM}}
        = \reel(
        \nsamples\alpha_{\densityGenerator}\tr(\tangentVector^{\mathcal{E}}_q\tangentVector^{\mathcal{E}}_\ell)
        +
        \nsamples\beta_{\densityGenerator}\tr(\tangentVector^{\mathcal{E}}_q)\tr(\tangentVector^{\mathcal{E}}_\ell)
        ).
    \end{equation*}
    Hence, from the relations
    \begin{equation*}
        \tr(\tangentVector^{\mathcal{E}}_q\tangentVector^{\mathcal{E}}_\ell) = \delta_{q\ell}
        ~~~~~~~~
        \text{and}
        ~~~~~~~~
        \tr(\tangentVector^{\mathcal{E}}_q)\tr(\tangentVector^{\mathcal{E}}_\ell) =
        \left\{
        \begin{array}{l}
            1 ~\text{if}~ (q,\ell) \in \llbracket 1,p \rrbracket^2 \\
            0 ~\text{otherwise},
        \end{array}
        \right.
    \end{equation*} 
    we obtain the Fisher information matrix
    \begin{equation*}
        \mathbf{F}_{(1,0)} = 
        \nsamples\alpha_{\densityGenerator}\mathbf{I}_{\nfeatures^2}
        + \nsamples\beta_{\densityGenerator}
        \begin{bmatrix}
        \mathbf{1}_{\nfeatures\times\nfeatures} & \mathbf{0}_{\nfeatures\times\nfeatures(\nfeatures-1)} \\
        \mathbf{0}_{\nfeatures(\nfeatures-1)\times\nfeatures} & \mathbf{0}_{\nfeatures(\nfeatures-1)\times \nfeatures(\nfeatures-1)}  
        \end{bmatrix}
        ,
    \end{equation*}
    which is expressed as $\mathbf{F}_{(1,0)}=\nsamples\alpha_{\densityGenerator}\mathbf{I}_{\nfeatures^2} + \nsamples\nfeatures\beta_{\densityGenerator}\mathbf{v}\mathbf{v}^T$ with unitary vector $\mathbf{v}=\frac{1}{\sqrt{\nfeatures}}\left[ ~\mathbf{1}_{\nfeatures} ~|~ \mathbf{0}_{\nfeatures(\nfeatures-1)} ~ \right]$, i.e. $\mathbf{v}^T\mathbf{v}=1$.
    Hence, the inverse of the Fisher information matrix can be obtained by the Sherman-Morrison formula.
    In particular, its vector of eigenvalues can be directly identified as $\frac1{\nsamples}\left[ (\alpha_{\densityGenerator}+\nfeatures\beta_{\densityGenerator})^{-1}, \alpha_{\densityGenerator}^{-1},\ldots,\alpha_{\densityGenerator}^{-1} \right]$ and summed to obtain its trace. Theorem~\ref{thm:icrlb} and~\eqref{eq:crb_on_dist} are then applied to conclude.
\end{proof}

\Remark{
Contrarily to the Euclidean case of Theorem~\ref{thm:eucl_crb_ces}, the bound on the natural Riemannian distance in Theorem~\ref{thm:icrb_nat} does not depend on the parameter $\point$.
This is generally a desirable property, as it offers an interpretation grounded solely on intrinsic dimensions of the problem.
Additionally, simulation examples in Section~\ref{sec:simus} show that assessing the error with such criterion (that is more in accordance with the nature of the parameter) can also reveal unexpected properties of the estimates.
}

\subsubsection{Fisher-Rao distance}
\label{sec:intrinsic_crlb_ces}

The Fisher-Rao distance refers to the geodesic distance associated with the Fisher information metric (cf. Section \ref{sec:geodesicx_exp_log_dist}).
A subtlety is that we voluntarily omit the dependency on $\nsamples$ of the Fisher information metric of Theorem~\ref{thm:fim_metric}, i.e., the bound will be obtained for using a generic metric in~\eqref{eq:AI_metric} with $\alpha=\alpha_{\densityGenerator}$ and $\beta=\beta_{\densityGenerator}$.
This distinction has two main reasons: 
$i$) it appears more logical to evaluate performance with a distance whose expression does not vary with the sample support of the scenario $\nsamples$;
$ii$) this allows us to also stress that, though identical, two metrics play a separate role in the derivations: one is inherent to the statistical model, the other is a choice made to measure estimation accuracy.
Hence, the elementary tools on $\manifold{\nfeatures}$~are
\begin{equation}
    \begin{array}{lcl}
        \text{Metric:} & ~ &
        \metric{\point}{\tangentVector}{\tangentVectorBis}^{(\alpha_\densityGenerator,\beta_\densityGenerator)} =
        \reel(
        \alpha_\densityGenerator \tr(\point^{-1}\tangentVector\point^{-1}\tangentVectorBis)
        + \beta_\densityGenerator \tr(\point^{-1}\tangentVector)\tr(\point^{-1}\tangentVectorBis)
        )
        \\[0.1cm]
        \text{Logarithm:} & ~ & 
        \log_{\point}(\estimator) \; = \; \point\log(\point^{-1}\estimator)
        \\[0.1cm]
        \text{Distance:} & ~ & 
        \delta_{(\alpha_\densityGenerator,\beta_\densityGenerator)}^2(\point,\estimator)
        = \alpha_\densityGenerator \|\log(\point^{-1}\estimator)\|_2^2
        + \beta_\densityGenerator (\log\det(\point^{-1}\estimator))^2.
    \end{array}
\label{eq:recall_fisher_metric}
\end{equation}
Recall that full details on this geometry are provided in Section~\ref{sec:RiemGeo}.
Contrary to previous geometries, since the considered metric is the Fisher information one, we do not actually need to compute a basis of the tangent space $\tangentSpace{\nfeatures}{\point}$ to obtain the bound.
However, notice that if needed, such a basis can be obtained using the Gram-Schmidt orthogonalization process.
In this case, the Cramér-Rao bound is: 
\begin{theorem}[Cramér-Rao bound on Fisher-Rao distance]
    Let $\estimator$ be an unbiased estimator of $\point\in\manifold{\nfeatures}$ built from iid data $\{\data_i\}_{i=1}^{\nsamples}$ drawn from $\data\sim\distribution{\point}{\densityGenerator}$.
    The Fisher-Rao distance between $\estimator$ and $\point$ is bounded in expectation as
    \begin{equation*}
        \expectation{\delta_{(\alpha_\densityGenerator,\beta_\densityGenerator)}^2(\estimator,\point)}
        \geq
        \frac{\nfeatures^2}{\nsamples}.
    \end{equation*}
\label{thm:icrb_ces}
\end{theorem}
\begin{proof}
    By definition, we have $\metric{\point}{\cdot}{\cdot}^{\textup{FIM}}=\nsamples\metric{\point}{\cdot}{\cdot}^{(\alpha_\densityGenerator,\beta_\densityGenerator)}$.
    Hence, since the basis of interest, denoted $\{\tangentVector^{(\alpha_\densityGenerator,\beta_\densityGenerator)}_q\}_{q=1}^{\nfeatures^2}$, is orthonormal according to $\metric{\point}{\cdot}{\cdot}^{(\alpha_\densityGenerator,\beta_\densityGenerator)}$, it follows that $\mathbf{F}_{(\alpha_\densityGenerator,\beta_\densityGenerator)}=\nsamples\mathbf{I}_{\nfeatures^2}$.
    The trace of its inverse is therefore $p^2/n$ and the proof is concluded by applying Theorem \ref{thm:icrlb} and \eqref{eq:crb_on_dist}.
\end{proof}
We notice that Theorems \ref{thm:icrb_nat} and \ref{thm:icrb_ces} coincide in the Gaussian case ($\alpha_g=1$ and $\beta_g=0$).

\Remark{
Theorem~\ref{thm:icrb_ces} actually exemplifies a more of universal result, which illustrates that the Fisher-Rao distance is the most in accordance with the underlying statistical model.
Indeed, the proof strategy of Theorem~\ref{thm:icrb_ces} holds for any geometry induced by a statistical model (parameter manifold and probability density function).
Thus, the Fisher-Rao distance will always be bounded by a ratio between the intrinsic problem dimension and the number of samples.
}

\subsection{Simulation examples}

\label{sec:simus}

This section illustrates the results of Theorems \ref{thm:eucl_crb_ces}-\ref{thm:icrb_ces} for the multivariate $t$-distribution (cf. example of Section \ref{sec:background_ces}), and various covariance matrix estimators.
In the following, the scatter matrix is built as a $\nfeatures\times\nfeatures$ (with $\nfeatures=10$) Toeplitz matrix $\left[\point_T\right]_{ij}=\rho^{|i-j|}$ with $\rho=\nicefrac{0.9(1+\sqrt{-1})}{\sqrt{2}}$.
For samples distributed as $\data\sim\distribution{\point_T}{\densityGenerator_d}$, where $\densityGenerator_d$ is the density generator of the $t$-distribution with $d$ degrees of freedom, we study the performance of the following estimators of $\point_T$:
\begin{itemize}
\item SCM: the usual sample covariance matrix, defined as $\estimator_{\textup{SCM}} = \frac1{\nsamples}\sum_{i=1}^{\nsamples} \data_i \data_i^H$.
\item MLE: The estimator $\estimator_{\textup{MLE}}$ defined in~\eqref{eq:MLE} using the appropriate function $\psi(t)=-\phi(t)$, with $\phi$ defined in~\eqref{eq:phi_mle_t}.
\item Mismatched MLE: the $M$-estimator $\estimator_{\textup{m-MLE}}$ constructed as the MLE, except that the parameter $d$ is different from the true parameter. Here, $d=10$ is set regardless of the underlying distribution.
\end{itemize}
These performances are evaluated with respect to $\nsamples$ ($\nsamples$ ranging from $11$ to $10^3$) through the mean squared distances $\delta_{\mathcal{E}}^2$, $\delta_{(1,0)}^2$ and $\delta_{(\alpha_{\densityGenerator_d},\beta_{\densityGenerator_d})}^2$ (evaluated on $10^4$ Monte-Carlo simulations) and are compared to the corresponding Cramér-Rao lower bounds from Theorems \ref{thm:eucl_crb_ces}-\ref{thm:icrb_ces}.

The left column of Figure \ref{fig:fig1} displays the results for a $t$-distribution with $d=100$ degrees of freedom.
Notice that, in this case, data almost follow a Gaussian distribution 
(it is usually admitted that $d>30$ allows to assume Gaussianity of the data).
In this setting, $\estimator_{\textup{MLE}}\simeq\estimator_{\textup{SCM}}$ so these estimators reach similar performances.
For all performance measurements (different distances), the mismatched MLE appears not efficient at high sample support, which is due to a bias induced on the scale through the wrong choice of parameter $d$. 
Also, $\alpha\simeq1$ and $\beta\simeq0$, so $\metric{\cdot}{\cdot}{\cdot}^{(1,0)}$ and $\metric{\cdot}{\cdot}{\cdot}^{(\alpha_{\densityGenerator_d},\beta_{\densityGenerator_d})}$ generate almost identical distances and corresponding bounds, as observed in Figure~\ref{fig:fig1}.
Interestingly, as noted in~\cite{smith2005covariance}, these performance criteria show that the studied estimators are not efficient at low sample support.
The natural metric is able to reflect some empirical results in terms of application -- the SCM is known to provide an inaccurate estimation at low sample support --, while the Euclidean metric is apparently not, i.e., the Cramér-Rao bound and MSE on the Euclidean metric appear non-informative here.

The right column of Figure~\ref{fig:fig1} displays the same results for a $t$-distribution with $d=3$ degrees of freedom. 
Here, the distribution is heavy tailed and the SCM, as well as the mismatched MLE, fail to provide an accurate estimator of the scatter matrix.
In this case, the study of the Euclidean metric reveals that the MLE is not efficient at low sample support, however it converges to the bound as $\nsamples$ grows.
We notice that the convergence towards this regime appears to be slower through the study of the natural and C-CES Fisher-Rao metric, which may be an interesting point in order to quantify the number of samples needed to achieve good performance in terms of application purpose.

\begin{figure}
    \centering
    \input{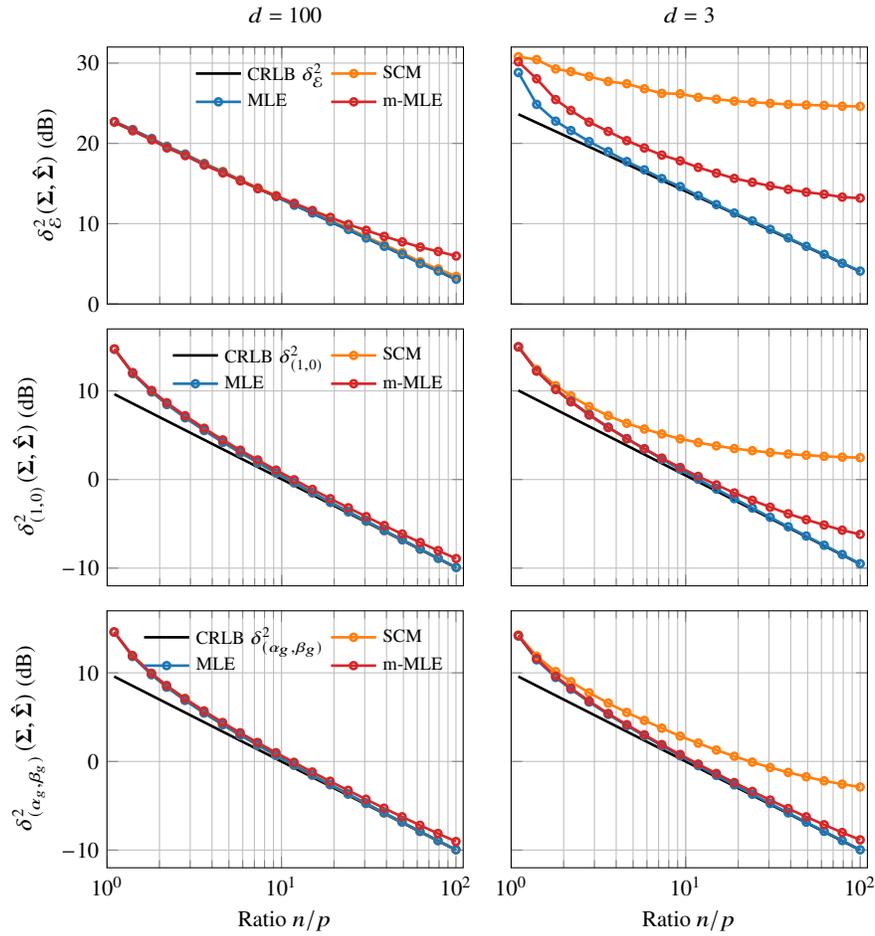}
    \caption{
    From top to bottom: Euclidean, Natural, CES Fisher-Rao CRLB and mean squared distance scatter matrix for t-distribution with $d$ degrees of freedom versus $\nsamples/\nfeatures$ for $\nfeatures=10$. On the left, $d=100$ (close to Gaussian case) and, on the right, $d=3$.
    }
    \label{fig:fig1}
\end{figure}


%

\section{Riemannian classification with the Fisher-Rao distance}
\label{sec:ClassifEEG}

Classification is a ubiquitous task in machine learning.
From a statistical point of view, the problem generally consists of attributing a class to each sample (or batch of samples) from an unlabelled mixture of different distributions.
The Fisher-Rao geometry provides a tool that can be efficiently leveraged in this context: as most classification methods are based on the Euclidean distance between samples, these can be transposed to the Riemannian setting by using the Fisher-Rao distance on the statistical feature space (i.e., the parameters of the assumed model).
Such transposition is often beneficial as it leverages a metric that is in accordance with the model (e.g., it can account for its natural geometric invariance).
In this regard, Section \ref{sec:classif_framework} presents a generic framework driven by the Fisher-Rao geometry.
An example based on CES models and the nearest centroïd classifier is derived in Section \ref{sec:Riem_mdm} and applied to EEG recordings in Section \ref{sec:appli_EEG}.

\subsection{A Fisher-Rao Riemannian classification framework}


\label{sec:classif_framework}

The use of statistical features (or descriptors) is common in batch sample classification, as these tend to be more discriminative than raw data.
Interestingly, when assuming a statistical model for the batches, the model parameters appear as a natural choice for such statistical features, and the Fisher-Rao distance as a natural tool to compare them.
For example, assuming two C-CES models with the same probability density function $f$, but different parameters $\point_1$ and $\point_2$, the Fisher-Rao distance (cf. Theorem \ref{thm:distance_ces} and \eqref{eq:recall_fisher_metric}) acts distance between statistical models through the following relation:
\begin{equation}
    \underbrace{
    \delta_{\rm FR}(
    f(\data|\point_1)
    ,
    f(\data|\point_2)
    )}_{\text{dist. between models}}
    \overset{\text{def}}{=}
    \underbrace{
    \delta_{\textup{FR}} ( \point_1, \point_2).}_{\text{FR-dist. between parameters}}
    \label{eq:DR_as_dist_between_models}
\end{equation}
In practice, we handle empirical distributions (i.e., batches of samples), so this distance can be evaluated as:
\begin{equation}
    \underbrace{
    \hat{\delta}_{\textup{FR}}(
    \{ \data_{i,1} \}_{i=1}^n
    ,
    \{ \data_{i,2} \}_{i=1}^n
    )}_{\text{dist. between batches}}
    \overset{\text{def}}{=}
    \underbrace{
    \delta_{\textup{FR}} ( \estimator_1, \estimator_2),}_{\text{FR-dist. between estimated covariances}}
    \label{eq:DR_as_dist_between_models2}
\end{equation}
where $\{ \data_{i,1} \}_{i=1}^n$ (resp. $\{ \data_{i,2} \}_{i=1}^n$)
denotes a sample batch, and $\estimator_1$ (resp. $\estimator_2$) 
denotes an estimate of its covariance matrix, such as the maximum likelihood estimator presented in Section \ref{sec:Optim}.
From this perspective, a batch classification problem then turns into a problem of classifying covariance matrices on $\manifold{\nfeatures}$.
Such a task can be achieved by using a standard classification algorithm in which criteria and objects are carefully transposed according to the Fisher-Rao distance (rather than the Euclidean one).
For examples related to this setup: the Riemannian nearest centroïd (or minimum distance to mean) classifier~\cite{barachant2011multiclass,tuzel2008pedestrian}; the Riemannian $K$-means on $\manifold{\nfeatures}$ was, e.g., used in \cite{collas2021probabilistic, hippert2022robust}, Kernel methods based on Riemannian distances were studied in \cite{barachant2013classification,jayasumana2016kernels,jayasumana2013kernel}, and Riemannian Gaussian mixture models on $\manifold{\nfeatures}$ were proposed in \cite{said2017riemannian,said2017gaussian}.
The following section presents the Riemannian counterpart of the nearest centroid classifier for $\manifold{\nfeatures}$.

\Remark{
Beyond C-CES models, the presented framework generalizes to a generic (model-driven) Riemannian classification methodology, which can be summarized as follows: 
$i$) Model selection: we assume an underlying statistical model, whose parameters should differ between classes;
$ii$) Statistical Feature extraction: we estimate the corresponding parameters for each batch;
$iii$) Riemannian classification: the extracted features are classified by leveraging the Fisher-Rao distance.
}

\subsection{Nearest centroïd classifier on $\manifold{\nfeatures}$ for C-CES models}

\label{sec:Riem_mdm}

The Riemannian center of mass corresponding to the framework discussed in Section \ref{sec:classif_framework} when assuming a Gaussian model has been the reference method to classify electroencephalography (EEG) recordings for the past decade~\cite{barachant2011multiclass}.
This section extends this methodology to the C-CES distributions, and presents the necessary tools to compute the Riemannian center of masses on $\manifold{\nfeatures}$.

Formally, we focus here on the supervised classification of batches of data.
Formally, given an unknown batch of a $\nsamples$-sample $\{\data_i\}_{i=1}^{\nsamples}$ and $z$ fixed classes, a classifier $\classifier:(\dataSpace{\nfeatures})^{\nsamples}\to\llbracket1,\dots,\nclasses\rrbracket$ infers the class label $y\in\llbracket1,\dots,\nclasses\rrbracket$, i.e.,
\begin{equation}
    y=\classifier\left(\{\data_i\}_{i=1}^{\nsamples}\right).
\end{equation}
To provide accurate results, the classifier $\classifier$ is trained on $\nbatches$ batches of $\nsamples$ samples $\{\{\data_{i,j}\}_{i=1}^{\nsamples},\labels_j\}_{j=1}^{\nbatches}$ associated to known class labels $\labels_j\in\llbracket1,\dots,\nclasses\rrbracket$.
In practice, one usually aims to evaluate the accuracy of a classifier on some dataset $\dataset$.
To do so, the dataset is split into training and test sets, denoted
$\train$ and $\test$, respectively.
The classifier $\classifier$ is trained on $\train$ and prediction is performed on the testing set $\test$.
Predicted labels are then compared to actual labels, which yields the accuracy of $\classifier$ on the considered dataset.
Notice that there are different ways to build $\train$ and $\test$ from $\dataset$, see e.g., the documentation of scikit-learn~\cite{pedregosa2011scikit} for more details.

For the model selection step, we consider that each batch $\{\data_{i,j}\}_{i=1}^{\nsamples}$ is distributed according to $\data\sim\distribution{\point_j}{\densityGenerator}$.
The statistical parameter extraction step is then performed by maximum likelihood estimation on each batch (cf. Section \ref{sec:Optim}).
From there, the feature classification problem is set as $\dataset=\{\pointBis_j,\labels_j\}_{j=1}^{\nbatches}$ on $\manifold{\nfeatures}$.
We then exploit the Fisher-Rao distance $\delta$ of C-CES distribution defined in Theorem~\ref{thm:distance_ces} to generalize the nearest centroïd classifier, also referred to as minimum distance to mean (MDM) classifier, to $\manifold{\nfeatures}$, .
This classification algorithm consists of two steps:
\begin{itemize}
    \item First, it computes the center of mass of each class, also called class center, from covariance matrices in the training set $\train$.
    \item Then, it assigns the label of the closest class center to each covariance in $\test$.
\end{itemize}
Since the covariance matrices lie on the Riemannian manifold $\manifold{\nfeatures}$, the geodesic distance $\delta$ from Theorem~\ref{thm:distance_ces} is leveraged in both steps.

We now detail the first step.
For every class $\labels\in\llbracket1,\dots,\nclasses\rrbracket$, one must compute the class center $\classCenter^{(\labels)}$ from the training set $\train$.
It is the center of mass of the set $\{\pointBis_j\in\train: \labels_j=\labels\}$.
We thus need to be able to compute the center of mass $\classCenter$ of a set $\{\point_j\}_{j=1}^{\nbatches}$ of matrices in $\manifold{\nfeatures}$ according to the Fisher-Rao distance in Theorem~\ref{thm:distance_ces}.
Following~\cite{karcher1977riemannian}, the Riemannian center of mass is defined in the following Definition~\ref{def:center_of_mass}.

\begin{definition}[Riemannian center of mass on $\manifold{\nfeatures}$]
    The center of mass $\classCenter^{\star}$ of $\{\point_i\}_{i=1}^{\nbatches}$ on $\manifold{\nfeatures}$ is defined as the minimizer of the variance computed with the geodesic distance
    \begin{equation}
        \classCenter^{\star} = \argmin_{\classCenter\in\manifold{\nfeatures}} \quad V(\classCenter) 
        \label{eq:riemannian_mean}
    \end{equation}
    with $ \quad V(\classCenter) \overset{\text{def}}{=} \frac{1}{2\nbatches} \sum_{j=1}^{\nbatches}  \squaredist{\classCenter}{\point_j}.$
    \label{def:center_of_mass}
\end{definition}
Remark that if the Riemannian distance $\delta$ is replaced by its Euclidean counterpart, $\delta_{\mathcal{E}}(\point,\pointBis)=\|\point-\pointBis\|_2$, then the minimizer of $V$ becomes the arithmetic mean $\classCenter=\frac1{\nbatches}\sum_{j=1}^{\nbatches}\point_j$.
Unfortunately, for the Riemannian case, a closed-form solution of~\eqref{eq:riemannian_mean} remains unknown~\cite{moakher2005differential} except in very specific cases ($\nbatches=2$, commuting matrices,~...).
Hence, one must turn to an iterative optimization procedure.
As in~\cite{pennec2006riemannian}, we focus here on a Riemannian gradient descent on $\manifold{\nfeatures}$.
Recall from Section~\ref{sec:Optim} that, to employ this algorithm, we need to compute the Riemannian gradient of~\eqref{eq:riemannian_mean}, choose a retraction and a step size rule.
The Riemannian gradient of $V$ was derived in~\cite{karcher1977riemannian,moakher2005differential}, and is provided in Proposition~\ref{prop:rgrad_mean}.

\begin{proposition}[Riemannian gradient of $V$]
    the Riemannian gradient $\grad V(\classCenter)$ of the variance $V$ defined in~\eqref{eq:riemannian_mean} at $\classCenter\in\manifold{\nfeatures}$ is
    \begin{equation}
        \grad V(\classCenter) = -\frac{1}{\nbatches} \sum_{j=1}^{\nbatches} \log_{\classCenter}(\point_j) = -\frac{1}{\nbatches} \sum_{j=1}^{\nbatches} \classCenter^{\nicefrac12} \log(\classCenter^{-\nicefrac12} \point_j \classCenter^{-\nicefrac12}) \classCenter^{\nicefrac12}.
        \label{eq:grad_mean}
    \end{equation}
\label{prop:rgrad_mean}
\end{proposition}
\begin{proof}
    In~\cite{moakher2005differential}, a technical proof directly deriving the distance of Theorem~\ref{thm:distance_ces} for $\alpha=1$ and $\beta=0$ is provided.
    Here, we propose a more general Riemannian geometry proof, which do not depend on the distance or the manifold.
    The proved result is well-known and can for instance be found in~\cite{pennec2017hessian} without proof.
    Given $\point\in\manifold{\nfeatures}$, we aim to show that the gradient of the function $v(\classCenter)=\frac12\squaredist{\classCenter}{\point}$ is $\grad v(\classCenter)=-\log_{\classCenter}(\point)$, where $\log_{\cdot}(\cdot)$ is the Riemannian logarithm mapping corresponding to the Riemannian distance $\delta(\cdot,\cdot)$.
    Let $\classCenter(t)$ the geodesic such that $\classCenter(0)=\classCenter$ and $\dot{\classCenter}(0)=\tangentVector\in\ambientSpace{\nfeatures}$.
    It follows that $\Diff v(\classCenter)[\tangentVector]=\left.\frac{\D}{\D t}v(\classCenter(t))\right|_{t=0}$.
    Let $\gamma_t$ the geodesic joining $\classCenter(t)$ to $\point$.
    By construction, $H(s,t)=\gamma_t(s)$ is a variation of the geodesic $\gamma_0$~\cite[Definition 3.24]{gallot1990riemannian}.
    Furthermore, we have $\left.\frac{\D}{\D t}v(\classCenter(t))\right|_{t=0}=\left.\frac{\D}{\D t}E(\gamma_t)\right|_{t=0}$, where $E(\gamma_t)=\frac12\int_0^1\langle\dot{\gamma}_t(s),\dot{\gamma}_t(s)\rangle_{\gamma_t(s)}\D s$ is the energy of the geodesic $\gamma_t$.
    Let $Y(s)$ such that $H(t,s)=\exp_{\gamma_0(s)}(tY(s))$.
    From~\cite[Theorem 3.31]{gallot1990riemannian}, we get the first variation formula of energy
    \begin{equation*}
        \frac{\D}{\D t}E(\gamma_t) = [\langle Y(s), \dot{\gamma}_0(s)\rangle_{\gamma_0(s)}]_0^1 - \int_0^1 \langle Y(s), \nabla_{\dot{\gamma}_0(s)}\dot{\gamma}_0(s)\rangle_{\gamma_0(s)} \D s.
    \end{equation*}
    Since $\gamma_0$ is a geodesic, $\nabla_{\dot{\gamma}_0(s)}\dot{\gamma}_0(s)=0$.
    Hence, the second term vanishes.
    Moreover, $Y(1)=\frac1t\log_{\gamma_0(1)}(\gamma_t(1))$.
    Since $\gamma_0(1)=\gamma_t(1)=\point$, $Y(1)=0$.
    We also have $Y(0)=\frac1t\log_{\gamma_0(0)}(\gamma_t(0))=\frac1t\log_{\classCenter}(\classCenter(t))=\frac1t t\tangentVector=\tangentVector$.
    It follows that
    \begin{equation*}
        \frac{\D}{\D t}E(\gamma_t) = - \langle \tangentVector, \dot{\gamma}_0(0) \rangle_{\gamma_0(0)} = \langle \tangentVector, -\log_{\classCenter}(\point) \rangle_{\classCenter}.
    \end{equation*}
    We thus get $\Diff v(\classCenter)[\tangentVector]=\langle \tangentVector, -\log_{\classCenter}(\point) \rangle_{\classCenter}$.
    The result follows by identification.
    One can then conclude the proof of the proposition by using the sum property of the gradient operator.
\end{proof}
Then, the most common choice for the retraction is to take the Riemannian exponential mapping~\eqref{eq:Rexp}.
Furthermore, the stepsize in this case is often simply set to $1$.
It follows that, given some initialization $\classCenter_{(0)}$, the sequence of iterates is
\begin{equation}
    \begin{array}{rcl}
         \classCenter_{(k+1)} & = & \exp_{\classCenter_{(k)}}\left( \frac1{\nbatches}\sum_{i=1}^{\nbatches} \log_{\classCenter_{(k)}}(\point_j) \right)
         \\[7pt]
         & = & \classCenter_{(k)}^{\nicefrac12}\exp\left( \frac1{\nbatches} \sum_{i=1}^{\nbatches} \log(\classCenter_{(k)}^{-\nicefrac12} \point_j \classCenter_{(k)}^{-\nicefrac12}) \right)\classCenter_{(k)}^{\nicefrac12}.
    \end{array}
\label{eq:riemannian_mean_GD}
\end{equation}
The variance $V$~\eqref{eq:riemannian_mean} is a strictly geodesically convex function over $\manifold{\nfeatures}$~\cite{tang2021CESmean}.
Hence, its minimizer is unique.

\Remark{Notice that there is no dependence on $\alpha$ and $\beta$ in~\eqref{eq:riemannian_mean_GD}.
This means that the Riemannian center of mass according to the Fisher-Rao distance in Theorem~\ref{thm:distance_ces} is the same for every C-CES distribution.
} 

The computation of the class centers being solved, we now turn to the second step of the nearest centroïd classifier: the assignment to a class $\labels_j$ of each estimated covariance matrix $\pointBis_j$ belonging to the test set $\test$.
This is achieved by taking the class that corresponds to the minimal geodesic distance with respect to all class centers, i.e.,
\begin{equation}
    \labels_j = \argmin_{\labels\in\llbracket1,\dots,\nclasses\rrbracket} \quad \left\{ \, \squaredist{\classCenter^{(\labels)}}{\pointBis_j} \, \right\}_{\labels\in\llbracket1,\dots,\nclasses\rrbracket}.
\label{eq:label_asignment}
\end{equation}
The resulting nearest centroïd classifier on $\HPD_\nfeatures$ is summarized in Algorithm~\ref{algo:nearest_centroid}.

\begin{algorithm}[h]
	\KwIn{A training set $\train=\{(\pointBis_j, \labels_j)\}_{j=1}^\nbatchestrain$ and a test set $\test=\{\pointBis_j\}_{j=1}^\nbatchestest$.}
	\KwOut{Predictions of the test set $\{\labels_j\}_{j=1}^\mtest$.}
	\# Training\\
	\For{$\labels=1$ to $\nclasses$}{
		Compute the center of mass $\classCenter^{(\labels)}$ of $\{\pointBis_j \in \train: \, \labels_j=\labels\}$ with~\eqref{eq:riemannian_mean_GD}.
	}
	\# Testing\\
	\For{$j=1$ to $\mtest$}{
		Assign $\pointBis_j$ to the class with the nearest class center $\classCenter^{(y)}$ with~\eqref{eq:label_asignment}. \\
	}
	\caption{Nearest centroïd classifier on $\manifold{\nfeatures}$}
	\label{algo:nearest_centroid}
\end{algorithm}

\Remark{
The Gaussian assumption allows recover the classification algorithm from~\cite{barachant2011multiclass}, as in this case: 
$i$) the maximum likelihood estimator is the sample covariance matrix $\estimator_j = \frac1{\nsamples} \sum_{i=1}^{\nsamples} \data_{i,j} \data_{i,j}^H$;
$ii$) $\alpha=1$ and $\beta=0$ in the Fisher-Rao distance $\delta$ of Theorem~\ref{thm:distance_ces}. 
}

\subsection{Application to EEG classification}

\label{sec:appli_EEG}

One usually needs to classify EEG recordings in the context of brain-computer interfaces (BCI), where a subject interacts with a computer through brain activity.
There are several paradigms for BCI based on EEG.
The three main ones are: steady-states visually evoked potentials (SSVEP)~\cite{kalunga2016online}, motor imagery (MI)~\cite{tangermann2012review}, and event-related potentials (ERP)~\cite{arico2014influence}.
This example focuses on ERP data, where subjects are exposed to some stimuli (most often a visual one).
These induce a signal response in the brain: the so-called P300, which is a positive wave occurring 300 ms after the stimulus.
An ERP dataset consists in a set of trials separated into two classes: a target class (TA), for which the subject is exposed to a stimulus; and a non-target class (NT), for which there is no stimulus.
More specifically, we consider the BNCI2014$\_$009 dataset~\cite{arico2014influence}, which is available on the MOABB platform\footnote{
    https://github.com/NeuroTechX/moabb --
    A standard benchmark platform for BCI.
}.
This dataset contains data from 10 subjects, with 3 sessions each.
Data were acquired on 16 electrodes at 256 Hz and bandpass filtered between 0.1 Hz and 20 Hz.
Recordings were then downsampled to 128 Hz.
Each session of each subject contains 1728 trials of $0.8$s: 288 target and 1440 non-target ones.
Hence, each dataset (one session of one subject) yields $\dataset=\{\Data_j,\labels_j\}_{j=1}^{\nbatches}$ in $\realSpace^{\nfeatures\times\nsamples}\times\{\textup{TA},\textup{NT}\}$, where $\nfeatures=16$, $\nsamples=102$, $\nbatches=1728$ and $\nclasses=2$.

To perform classification of ERPs, raw data are not directly used.
Instead, following~\cite{barachant2014plug}, augmented data are leveraged.
Given the training set $\train=\{\Data_j\}_{j=1}^{\nbatchestrain}$, we compute the average target ERP with
\begin{equation}
    \mathbf{P}_{\textup{TA}} = \frac1{\nbatches_{\textup{TA}}} \sum_{\underset{\labels_j=\textup{TA}}{j=1}}^{\nbatchestrain} \Data_j,
\end{equation}
where $\nbatches_{\textup{TA}}$ is the number of target trials in the training set $\train$.
From there, augmented trials are defined as
\begin{equation}
    \AugmentedData_j = 
    \begin{bmatrix}
        \mathbf{P}_{\textup{TA}}
        \\
        \Data_j
    \end{bmatrix}.
\end{equation}
Covariance matrices $\pointBis_j$ are then estimated from these augmented trials both in the training and testing sets.
Finally, the nearest centroïd classifier in Algorithm~\ref{algo:nearest_centroid} is applied on these augmented covariance.
We compare two different versions here:
\begin{enumerate}
    \item \textbf{Gaussian version}: covariance matrices estimated through the sample covariance matrix (SCM) and nearest centroïd classifier employed with $\alpha=1$ and $\beta=0$.
    \item \textbf{$t$-distribution version}: covariance matrices estimated with the MLE of the $t$-distribution with $\nu=2.1$ degrees of freedom and nearest centroïd classifier used with $\alpha=\frac{\nu+\nfeatures}{\nu+\nfeatures+2}$ and $\beta=\alpha-1$.
\end{enumerate}
%


Achieved accuracies are presented in Figure~\ref{fig:MDM_BNCI}.
One can observe that both classifiers feature very good performance on this dataset.
One can further notice that they have very similar performance.
Indeed, on average, the nearest centroïd classifier with the $t$-distribution is better by $0.12\%$.
Considering that the SCM is much simpler to compute than the MLE of the $t$-distribution, one can argue that the nearest centroïd classifier associated with the Gaussian distribution is more advantageous on this dataset.
Due to the biological nature of the data, which can be expected to be noisy and contain a non-negligible amount of outliers, one could have expected that a heavy-tail distribution such as the $t$ with $\nu=2.1$ perform significantly better.
However, the dataset at hand has been curated and the preprocessing has been designed for the Gaussian distribution to work well.
Leveraging the $t$-distribution might be advanategous on real world non-curated data.

\begin{figure}
    \setlength\height{6.0cm} 
\setlength\width{0.8\linewidth}

\begin{center}
\begin{tikzpicture}

\begin{axis}[
        width  =\width,
        height =\height,
        at     ={(0,0)},
        xmin   = 0,
        xmax   = 4.5,
        xtick = {1.25,3.25},
        xticklabels = {MDM Gaussian, MDM Student $t$},
        ylabel = {ROC AUC},
        ymin   = 0.5,
        ymax   = 1,
        ytick={0.5,0.6,0.7,0.8,0.9,1},
        ymajorgrids,
    ]
    \addplot[only marks, mark size=1.3pt, mark options={solid, myorange,opacity=0.3}] table [x expr={1+0.015*\thisrowno{0}},y=MDM_Gaussian,col sep=comma] {figures/5_ClassifEEG/results.csv};

    \addplot[only marks, mark size=1.3pt, mark options={solid, myblue,opacity=0.3}] table [x expr={3+0.015*\thisrowno{0}},y=MDM_Student,col sep=comma] {figures/5_ClassifEEG/results.csv};

    \addplot[only marks, mark size=2pt, mark options={solid, myorange},error bars/.cd,y dir=both,y explicit,error bar style={color=myorange,line width=1pt}] table [x expr={1.25},y=MDM_Gaussian_mean,y error=MDM_Gaussian_std,col sep=comma] {figures/5_ClassifEEG/stats.csv};

    \addplot[only marks, mark size=2pt, mark options={solid, myblue},error bars/.cd,y dir=both,y explicit,error bar style={color=myblue,line width=1pt}] table [x expr={3.25},y=MDM_Student_mean,y error=MDM_Student_std,col sep=comma] {figures/5_ClassifEEG/stats.csv};
    
\end{axis}

\end{tikzpicture}
\end{center}
    \caption{AUC of ROC plots for the nearest centroïd classifiers exploiting the Gaussian distribution (left) and Student $t$-distribution with $\nu=2.1$ degrees of freedom (right) applied on the BNCI2014$\_$009 dataset~\cite{arico2014influence} (10 subjects, 3 sessions each).}
    \label{fig:MDM_BNCI}
\end{figure}
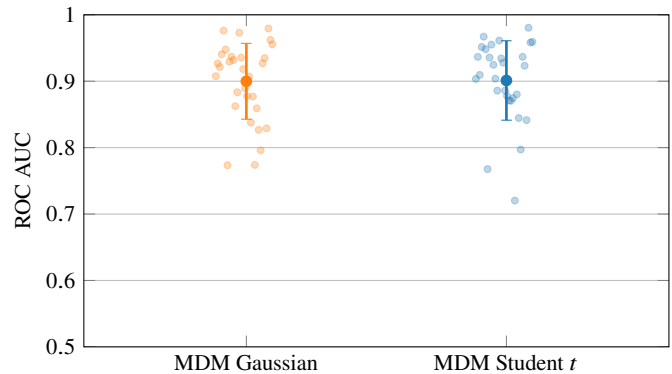


\section{Conclusion}
\label{sec:extension}

This chapter presented the Fisher-Rao geometry of C-CES distributions, and its practical uses in statistical signal processing and machine learning.
Remark that the methodology that consists in obtaining a Riemannian geometry from the Fisher information metric generalizes to any statistical model (assuming that the parameter space is a smooth manifold).
Hence, the approaches presented in this introduction can extend to many other models and applications.
Among other examples, such intrinsic analysis has been conducted for the estimation of rotations matrices \cite{boumal2014cramer} and for other Lie groups related to tracking problems \cite{labsir2021joint, labsir2023barankin}.
In other scopes more directly related to elliptical distributions, we can also mention that geometric tools were used for:
\begin{itemize}
    \item \textbf{Structured covariance matrices}: 
    In many applications, the covariance matrix is known to satisfy some form of structural constraint, that can be exploited to reduce the dimension of the estimation problem (see, e.g., \cite{wiesel2015structured, sun2016robust}).
    Geometric tools can then be leveraged by expressing the constrained space as a sub-manifold of $\HPD_{\nfeatures}$.
    For example: the Fisher information metric was used to obtain structured estimators in \cite{meriaux2019robust, meriaux2020matched};
    A geometry of Toeplitz matrices was studied in \cite{arnaudon2013riemannian};
    A framework for in probabilistic component analysis (low-rank structured covariance matrices) in C-CES was proposed in \cite{bouchard2021riemannian};
    Kronecker products preserve geodesic convexity \cite{Wiesel2012geodesic},
    ans such structure was considered in online covariance matrix estimation in~\cite{bouchard2021line}; geometry and structured covariance have also been considered for blind source separation~\cite{bouchard2021riemannian2}.

    \item \textbf{Non-centered models}: 
    The geodesics and Fisher-Rao distance of the model $\mathbf{x}\sim \mathcal{CES}(\boldsymbol{\mu},\mathbf{\Sigma})$ for the mean-and-covariance product manifold $\mathbb{C}^\nfeatures \times \HPD_\nfeatures$ remains intractable in the general case.
    Even for the Gaussian distribution $\mathbf{x}\sim \mathcal{N}(\boldsymbol{\mu},\mathbf{\Sigma})$, only special cases and approximations from geodesic triangles can be obtained \cite{calvo1991explicit, tang2018information, collas2022use}.
    Numerical methods to evaluate these geodesics and corresponding distances were proposed in \cite{nielsen2023simple, nielsen2023fisher}.
    Concerning estimation problems, Riemannian optimization was leveraged for non-centered mixture of scaled Gaussian distributions (a sub-family of C-CES distributions) in \cite{collas2021tyler, collas2023riemannian}.

    \item \textbf{Mixture models}: 
    Mixtures of C-CES can occur within the samples (the observation is the sum of multiple independent contributions) or within batches (the sample set aggregating multiple classes of C-CES).
    The within-sample mixture is typically used to cast robust models for probabilistic principal component analysis \cite{chen2009robust, sun2015low, hong2021heppcat}.
    In this context, geometric tools were developed for low-rank scaled Gaussian signal corrupted by white Gaussian noise in~\cite{collas2021probabilistic}.
    The within-batch mixture corresponds to a typical sample-wise classification problem.
    For this purpose, $g$-convex relaxations for Gaussian mixture models were studied in~\cite{hosseini2015matrix}.

\end{itemize}
As a final note, we also point out that \textit{information geometry} also refers to a much broader field than the scope covered by this chapter \cite{amari2016information, amari2021information}.
For comprehensive overviews of the many geometric structures behind families of probability distributions, we refer the readers to \cite{nielsen2020elementary, nielsen2022many}.

%
%
\bibliographystyle{spmpsci}
\bibliography{biblio}

\end{document}